\documentclass[a4paper]{article}
\usepackage{fullpage,url,amssymb,amsmath,amsthm}
\usepackage[linesnumbered,ruled,vlined]{algorithm2e}
\usepackage{graphicx,caption,subcaption}
\graphicspath{{figures/}}
\usepackage{ifpdf}
\ifpdf  \fi
\usepackage{hyperref}

\def\baralpha{{\bar\alpha}}
\def\sign{{\mathrm{sign}}}
\def\HS{{\mathrm{HS}}}
\def\HE{{\mathrm{HE}}}
\def\CS{{\mathrm{CS}}}

\def\lhs{\mathrm{lhs}}
\def\rhs{\mathrm{rhs}}
\def\dx{\mathrm{d}x}

\def\calX{\mathcal{X}}
\def\calM{\mathcal{M}}
\def\calF{\mathcal{F}}
\def\Holder{\mathtt{H}}
\def\diag{\mathrm{diag}}
\def\inner#1#2{\langle #1,#2\rangle}
\def\tp{\tilde{p}}
\def\tq{\tilde{q}}
\def\inner#1#2{{\langle #1, #2 \rangle}}
\def\dmu{\mathrm{d}\mu}
\def\KL{\mathrm{KL}}
\DeclareMathOperator*{\argmin}{\arg\,\min}
\DeclareMathOperator*{\argmax}{\arg\,\max}

\newtheorem{example}{Example}
\newtheorem{lemma}{Lemma}
\newtheorem{fact}{Fact}
\newtheorem{definition}{Definition}
\newtheorem{remark}{Remark}

\begin{document}

\sloppy

\title{On H\"older projective divergences\footnote{Reproducible source code available at \url{https://www.lix.polytechnique.fr/\textasciitilde nielsen/HPD/}}}

\author{Frank Nielsen\footnote{Contact author: \texttt{Frank.Nielsen@acm.org}}\\
\'Ecole Polytechnique, France\\
Sony Computer Science Laboratories, Japan\\
\and{}Ke Sun\\
Computer, Electrical and Mathematical Sciences and Engineering Division\\
King Abdullah University of Science and Technology (KAUST), Saudi Arabia
\and{}St\'ephane Marchand-Maillet\\
Viper Group, Computer Vision and Multimedia Laboratory\\
University of Geneva, Switzerland
}
\date{}
\maketitle

\begin{abstract}
We describe a framework to build distances by measuring the tightness of inequalities, and introduce the notion of proper statistical divergences and improper pseudo-divergences.
We then consider the H\"older ordinary and reverse inequalities, and present two novel classes of H\"older divergences and pseudo-divergences that both encapsulate the special case of the Cauchy-Schwarz divergence.
We report closed-form formulas for those statistical dissimilarities when considering distributions belonging to the same exponential family provided that the natural parameter space is a cone (e.g., multivariate Gaussians), or affine (e.g., categorical distributions).
Those new classes of H\"older distances are invariant to rescaling, and thus do not require distributions to be normalized.
Finally, we show how to compute statistical H\"older centroids with respect to those divergences, and carry out center-based clustering toy experiments on a set of  Gaussian distributions
that demonstrate empirically that symmetrized H\"older divergences outperform the symmetric Cauchy-Schwarz divergence.
\end{abstract}


\section{Introduction: Inequality, proper divergence and improper pseudo-divergence}\label{sec:intro}

\subsection{Statistical divergences from inequality gaps}
An inequality is denoted mathematically by  $\lhs\leq \rhs$, where $\lhs$ and $\rhs$ denote respectively the {\em left-hand-side} and {\em right-hand-side}
of the inequality.
One can build {\em dissimilarity} measures from inequalities $\lhs\leq \rhs$ by measuring the {\em inequality tightness}: 
For example, we may quantify the tightness of an inequality by its {\em difference  gap}: 
\begin{equation}
\Delta = \rhs-\lhs \geq 0.
\end{equation}

When $\lhs>0$, the inequality tightness can also be gauged by the {\em log-ratio gap}: 
\begin{equation}
D = \log\frac{\rhs}{\lhs}=-\log\frac{\lhs}{\rhs} \geq 0.
\end{equation}

We may further compose this inequality tightness value measuring non-negative gaps with a {\em strictly monotonically
increasing function} $f$ (with $f(0)=0$).

A {\em bi-parametric} inequality $\lhs(p,q)\le\rhs(p,q)$ is called {\em proper}
if it is strict for $p\neq{}q$ (i.e., $\lhs(p,q)<\rhs(p,q), \forall p\neq{}q$)
and {\em tight} if and only if (iff) $p=q$  (i.e., $\lhs(p,q) = \rhs(p,q), \forall p=q$).
Thus a proper bi-parametric inequality allows one to define dissimilarities such that $D(p,q)=0$ iff  $p=q$.
Such a dissimilarity is called proper. Otherwise, an inequality or dissimilarity is said {\em improper}.
Note that there are many equivalent words used in the literature instead of (dis-)similarity:
distance (although often assumed to have metric properties), pseudo-distance, discrimination, proximity, information deviation, etc. 

A {\em statistical dissimilarity} between two discrete or continuous distributions $p(x)$ and $q(x)$ on a support $\calX$ 
can thus be defined from inequalities by summing up or taking the integral for the inequalities instantiated on the observation space $\calX$:
\begin{align}
\forall x\in\calX,\quad &D_x(p,q)=\rhs(p(x),q(x))-\lhs(p(x),q(x)) \Rightarrow\\
&D(p,q) = \left\{
\begin{array}{ll}
\sum_{x\in\cal X} \big[ \rhs(p(x),q(x))-\lhs(p(x),q(x)) \big]& \mbox{discrete case},\\
\int_{\cal X} \big[ \rhs(p(x),q(x))-\lhs(p(x),q(x)) \big]\dx & \mbox{continuous case}.
\end{array}
\right.
\end{align}
In such a case, we get a {\em separable divergence}.
Some non-separable inequalities induce a {\em non-separable divergence}.
For example, the renown Cauchy-Schwarz divergence~\cite{CSRepDataSampling-2011} is not separable because in the inequality:
\begin{equation}
\int_\calX p(x)q(x)\dx \leq \sqrt{\left(\int_\calX p(x)^2 \dx\right)\left(\int_\calX q(x)^2 \dx\right)},
\end{equation}
the rhs is not separable.

Furthermore,  a proper dissimilarity is called a {\em divergence} in information
geometry~\cite{IG-2016} when it is $C^3$ (i.e., three times differentiable thus
allowing to define a metric tensor~\cite{rao45} and a cubic tensor~\cite{IG-2016}).

Many familiar distances can be reinterpreted as inequality gaps in disguise. 
For example, Bregman divergences~\cite{cb-2005} and Jensen divergences~\cite{tJ-2015} (also called Burbea-Rao divergences~\cite{burbea-1982,BR-2011}) can be reinterpreted as inequality difference gaps,
 and the Cauchy-Schwarz distance~\cite{CSRepDataSampling-2011} as an inequality log-ratio gap:

\begin{example}[Bregman divergence as a Bregman score-induced gap divergence]
A proper score function~\cite{spsr-2007} $S(p:q)$ induces a gap divergence $D(p:q)=S(p:q)-S(p:p)\geq 0$.
A Bregman divergence~\cite{cb-2005} $B_F(p:q)$ for a strictly convex and differentiable real-valued generator $F(x)$ is induced by the {\em Bregman score} 
$S_F(p:q)$.  
Let $S_F(p:q)=-F(q)-\inner{p-q}{\nabla F(q)}$ denote the Bregman proper score minimized for $p=q$.
Then the Bregman divergence is a gap divergence: $B_F(p:q)=S_F(p:q)-S_F(p:p)\geq 0$.
When $F$ is strictly convex, the Bregman score is proper, and the Bregman divergence is proper.
\end{example}

\begin{example}[Cauchy-Schwarz distance as a log-ratio gap divergence]
Consider the Cauchy-Schwarz inequality $\int_\calX p(x)q(x)\dx \leq \sqrt{\left(\int_\calX p(x)^2 \dx\right)\left(\int_\calX q(x)^2 \dx\right)}$.
Then the Cauchy-Schwarz distance~\cite{CSRepDataSampling-2011}  between two continuous distributions is defined by $\CS(p(x):q(x))=-\log\frac{\int_\calX p(x)q(x)\dx}{\sqrt{(\int_\calX p(x)^2 \dx)(\int_\calX q(x)^2 \dx)}}\geq 0$.
\end{example}

Note that we use the modern notation $D(p(x):q(x))$ to emphasize that the divergence is potentially asymmetric: $D(p(x):q(x))\not =D(q(x):p(x))$, see~\cite{IG-2016}.
In information theory~\cite{cover-2012}, the older notation ``$||$'' is often used instead of ``$:$'' that is used in information geometry~\cite{IG-2016}.

To conclude this introduction, let us finally introduce the notion of {\em projective statistical distances}.
A statistical distance $D(p(x):q(x))$ is said projective when:
\begin{equation}
D(\lambda p(x):\lambda' q(x)) = D(p(x):q(x)),\quad \forall \lambda,\lambda'>0.
\end{equation}

The Cauchy-Schwarz distance is a projective divergence. 
Another example of such a projective divergence is the parametric $\gamma$-divergence~\cite{GammaDiv-2008}.

\begin{example}[$\gamma$-divergence as a projective score-induced gap divergence]
The $\gamma$-divergence~\cite{GammaDiv-2008,PM-ProjDiv-2016} $D_\gamma(p(x):q(x))$ for $\gamma>0$ is projective:
\begin{align*}
D_{\gamma}(p(x),q(x)) &=  S_{\gamma}(p(x),q(x))-S_{\gamma}(p(x),p(x)), \mbox{with}\\
S_{\gamma}(p(x),q(x)) &=-\frac{1}{\gamma(1+\gamma)}\frac{\int p(x)q(x)^\gamma\dx}{\left(
\int q(x)^{1+\gamma} \dx
\right)^{\frac{\gamma}{1+\gamma}}}.
\end{align*}
The $\gamma$-divergence is related to the proper pseudo-spherical score~\cite{GammaDiv-2008}.
\end{example}
 
The $\gamma$-divergences have been proven useful for robust statistical inference~\cite{GammaDiv-2008} in the presence of heavy outlier contamination.

\subsection{Pseudo-divergences and the axiom of indiscernability}

Consider a broader class of \emph{statistical
pseudo-divergences} based on \emph{improper inequalities}, where
the tightness of $\lhs(p,q)\le\rhs(p,q)$ does not imply that $p=q$.
This family of dissimilarity measures have interesting
properties which have not been studied before.

Formally, statistical pseudo-divergences are defined with respect to density
measures $p(x)$ and $q(x)$ with $x\in\calX$, where $\calX$ denotes the support. By
definition, pseudo-divergences satisfy the following three fundamental properties:

\begin{enumerate}
\item Non-negativeness: $D(p(x):q(x)) \ge 0$ for any $p(x), q(x)$;

\item Reachable indiscernability:
\begin{itemize}
\item $\forall{p(x)}$, there exists $q(x)$ such that $D(p(x):q(x))=0$,
\item $\forall{q(x)}$, there exists $p(x)$ such that $D(p(x):q(x))=0$.
\end{itemize}

\item Positive correlation: If $D(p:q)=0$, then $\left( p(x_1)-p(x_2) \right) \left( q(x_1)-q(x_2) \right) \ge0$ for any $x_1$, $x_2\in\calX$.
\end{enumerate}

As compared to \emph{statistical divergence} measures such as the Kullback-Leibler (KL)
divergence:
\begin{equation}
\KL(p(x):q(x))=\int_\calX p(x)\log\frac{p(x)}{q(x)}\dx,
\end{equation}
pseudo-divergences do not require $D(p(x):p(x))=0$. Instead, any pair of distributions 
$p(x)$ and $q(x)$ with
$D(p(x):q(x))=0$ only has to be ``positively correlated'' such that
$p(x_1)\le{}p(x_2)$ implies $q(x_1)\le{}q(x_2)$, and vice versa.  Any divergence
with $D(p(x):q(x))=0\Rightarrow{}p(x)=q(x)$  (law of indiscernibles) automatically satisfies
this weaker condition, and therefore any divergence belongs to the broader class of
pseudo-divergences. Indeed, if $p(x)=q(x)$ then $( p(x_1)-p(x_2) ) ( q(x_1)-q(x_2) )=(
p(x_1)-p(x_2) )^2 \ge 0$.  However the converse is not true. 
As we shall describe in the remainder, the family of
pseudo-divergences is {\em not} limited to proper divergence measures. 
In the remainder, the term ``pseudo-divergence''
refers to such divergences that are \emph{not} proper divergence measures.

We  study two novel statistical dissimilarity families: One family of statistical improper pseudo-divergences and
one family of proper statistical divergences.
Within the class of pseudo-divergences, this work concentrates on defining a
one-parameter family of dissimilarities called H\"older log-ratio gap divergence
that we concisely abbreviate as HPD for ``H\"older pseudo divergence'' in the remainder.
We also study its proper divergence counterpart termed HD for ``H\"older divergence.''

\subsection{Prior work and contributions}

The term ``H\"older divergence'' has first been coined in 2014 based on the definition of the {\em H\"older score}~\cite{affineInvariantDivergence-2014,scaleinvariantDiv-2014}:
The score-induced H\"older divergence $D(p(x):q(x))$ is a proper gap divergence that yields a scale-invariant divergence~\cite{affineInvariantDivergence-2014,scaleinvariantDiv-2014}.
Let $p_{a,\sigma}(x)=a\sigma p(\sigma x)$ for $a,\sigma>0$ be a transformation. Then a scale-invariant divergence~\cite{affineInvariantDivergence-2014,scaleinvariantDiv-2014}
satisfies $D(p_{a,\sigma}(x):q_{a,\sigma}(x))=\kappa(a,\sigma)D(p(x):q(x))$ for a function $\kappa(a,\sigma)>0$.
This gap divergence is proper since it is based on the so-called H\"older score~\cite{affineInvariantDivergence-2014,scaleinvariantDiv-2014} but is {\em not} projective and does not include the Cauchy-Schwarz divergence.
Due to these differences the H\"older log-ratio gap divergence introduced here shall not be confused with the
H\"older gap divergence induced by the {\em H\"older
score}~\cite{scaleinvariantDiv-2014,affineInvariantDivergence-2014} that relies
both on a scalar $\gamma$ and a function $\phi(\cdot)$.

We shall introduce {\em two} novel families of log-ratio projective gap divergences based on H\"older ordinary (or forward) and reverse inequalities that extend the Cauchy-Schwarz divergence, study their properties, and consider as an application clustering Gaussian distributions:
 We experimentally show better clustering results when using symmetrized H\"older divergences than using the Cauchy-Schwarz divergence.
To contrast with the ``H\"older composite score-induced divergences'' of~\cite{scaleinvariantDiv-2014}, our H\"older divergences admit closed-form expressions
between distributions belonging to the same exponential families~\cite{EF-2009} provided that the natural parameter space is a cone or affine.

Our main contributions are summarized as follows:

\begin{itemize}
\item Define the uni-parametric family of H\"older improper pseudo-divergences (HPDs) in \S\ref{sec:hpd} and the bi-parametric family of H\"older proper divergences
 in \S\ref{sec:hd} (HDs) for positive and probability measures,  and study their properties (including their relationships with skewed Bhattacharrya distances~\cite{BR-2011} via escort distributions);

\item Report closed-form expressions of those divergences for exponential families when the natural parameter space is a cone or affine (include but not limited to the cases of categorical distributions and multivariate Gaussian distributions) in \S\ref{sec:HCEF};

\item Provide approximation techniques to compute those divergences between mixtures based on log-sum-exp inequalities in \S\ref{sec:hmm};

\item Describe a variational center-based clustering technique based on the convex-concave procedure for computing  H\"older centroids,
and report our experimental results in \S\ref{sec:centroid}.
\end{itemize}

\subsection{Organization}
This paper is organized as follows: \S\ref{sec:hpd} introduces the definition
and  properties of H\"older pseudo-divergences (HPDs). It is followed by \S\ref{sec:hd} that describes H\"older  proper divergences (HDs).
In \S\ref{sec:HCEF}, closed-form expressions for those novel families of divergences are reported for the categorical, multivariate Gaussian, Bernoulli, Laplace and Wishart distributions.
\S\ref{sec:centroid} defines H\"older statistical centroids and presents a variational $k$-means clustering technique: We show experimentally that 
using  H\"older divergences improve over the Cauchy-Schwarz divergence.
Finally, \ref{sec:concl} concludes this work and hints at further perspectives from the viewpoint of statistical estimation and manifold learning.
In Appendix~\ref{sec:hdproof}, we recall the proof of the ordinary and reverse H\"older's inequalities.

\section{H\"older pseudo-divergence: Definition and properties}\label{sec:hpd}
H\"older's inequality (see~\cite{Holder-1889} and Appendix \ref{sec:hdproof} for a proof) states for
positive real-valued functions\footnote{In a more general form, H\"older's
inequality holds for any real and complex valued functions.
In this work, we only focus on real positive functions that are densities of positive measures.}
$p(x)$ and $q(x)$ defined on the support $\calX$ that:
\begin{equation}
\int_\calX p(x)q(x)\dx\le
\left(\int_\calX p(x)^\alpha\dx\right)^{\frac{1}{\alpha}}
\left(\int_\calX q(x)^\beta\dx\right)^{\frac{1}{\beta}},
\end{equation}
where exponents $\alpha$ and $\beta$ satisfy $\alpha\beta>0$ as well as
the {\em exponent conjugacy} condition: $\frac{1}{\alpha}+\frac{1}{\beta}=1$.
We also write $\beta=\bar{\alpha}=\frac{\alpha}{\alpha-1}$ meaning that
$\alpha$ and $\beta$ are {\em conjugate} H\"older exponents.
We check that $\alpha>1$ and $\beta>1$.
H\"older inequality holds even if the lhs is infinite (meaning that the integral diverges) since the rhs is also infinite in that case.

The \emph{reverse H\"older inequality} holds for conjugate
exponents $\frac{1}{\alpha}+\frac{1}{\beta}=1$ with $\alpha\beta<0$ (then
$0<\alpha<1$ and $\beta<0$, or $\alpha<0$ and $0<\beta<1$):

\begin{equation}
\int_\calX p(x)q(x)\dx\ge
\left(\int_\calX p(x)^\alpha\dx\right)^{\frac{1}{\alpha}} \left(\int_\calX q(x)^\beta\dx\right)^{\frac{1}{\beta}}.
\end{equation}
Both H\"older's inequality and the reverse H\"older inequality
turn tight when $p(x)^\alpha\propto{}q(x)^\beta$ (see proof in Appendix~\ref{sec:hdproof}).

\subsection{Definition}

Let $(\calX,\calF,\mu)$ be a measurable space where $\mu$ is the Lebesgue measure, and let $L^\gamma(\calX,\mu)$ denote the Lebesgue space of functions 
that have their $\gamma$-th power of absolute value Lebesgue integrable, for any
$\gamma>0$  (when $\gamma\geq 1$, $L^\gamma(\calX,\mu)$ is a Banach space). 
We define the following pseudo-divergence:

\begin{definition}[H\"older statistical pseudo-divergence, HPD]
For conjugate exponents $\alpha$ and $\beta$ with $\alpha\beta>0$,
the {\em H\"older Pseudo-Divergence} (HPD) between two densities $p(x)\in L^\alpha(\calX,\mu)$ and $q(x)\in L^\beta(\calX,\mu)$ of positive measures absolutely continuous  with respect to (wrt.) $\mu$ is defined by the following log-ratio gap:
\begin{equation}
D^\Holder_{\alpha}(p(x):q(x))
=-\log\left(\frac{\int_\calX p(x)q(x)\dx}{\left(\int_\calX p(x)^\alpha\dx\right)^{\frac{1}{\alpha}}
\left(\int_\calX q(x)^\beta\dx\right)^{\frac{1}{\beta}}}\right).
\end{equation}

When $0<\alpha<1$ and $\beta=\bar{\alpha}=\frac{\alpha}{\alpha-1}<0$, or
$\alpha<0$ and $0<\beta<1$, the {\em reverse HPD} is defined by:
\begin{equation}
D^\Holder_{\alpha}(p(x):q(x))= \log\left(\frac{\int_\calX p(x)q(x)\dx}
{\left(\int_\calX p(x)^\alpha\dx\right)^{\frac{1}{\alpha}}
\left(\int_\calX q(x)^\beta\dx\right)^{\frac{1}{\beta}}}\right).
\end{equation}
\end{definition}

By H\"older's inequality and the reverse H\"older inequality,
$D^\Holder_\alpha(p(x):q(x))\ge0$ with $D^\Holder_\alpha(p(x):q(x))=0$ iff $p(x)^\alpha\propto{q(x)}^\beta$
or equivalently $q(x)\propto{p(x)}^{\alpha/\beta}=p(x)^{\alpha-1}$.
When $\alpha>1$, $x^{\alpha-1}$ is monotonically increasing, and $D^\Holder_\alpha$ is indeed a pseudo-divergence.
However, the reverse HPD is \emph{not} a pseudo-divergence
because $x^{\alpha-1}$ will be monotonically decreasing if $\alpha<0$ or $0<\alpha<1$.
Therefore we only consider HPD with $\alpha>1$ in the remainder,
and leave here the notion of reverse H\"older divergence.

When $\alpha=\beta=2$, the HPD becomes the Cauchy-Schwarz divergence $\CS$~\cite{CSShapeMatching-2014}:
\begin{equation}
D^\Holder_{2}(p(x):q(x))=\CS(p(x):q(x))=
-\log\left(\frac{\int_\calX p(x)q(x)\dx}{\left(\int_\calX p(x)^2\dx\right)^{\frac{1}{2}}
\left(\int_\calX q(x)^2\dx\right)^{\frac{1}{2}}}\right),
\end{equation}
which has been proved useful to get closed-form divergence formulas between
mixtures of exponential families with conic or affine natural parameter
spaces~\cite{DivMix-2012}.

The Cauchy-Schwarz divergence is proper for probability densities
since the Cauchy-Schwarz inequality becomes an equality iff $q(x)=\lambda p(x)^{\alpha-1}=\lambda p(x)$ implies that
$\lambda=\int_\calX \lambda p(x) \dx=\int_\calX q(x) \dx=1$.
It is however not proper for positive densities.

\begin{fact}[CS is only proper for probability densities]
The Cauchy-Schwarz divergence $\CS(p(x):q(x))$ is proper for square-integrable probability densities $p(x), q(x)\in L^2(\calX,\mu)$ but not proper for positive square-integrable densities.
\end{fact}

\subsection{Properness and improperness}
In the general case, when $\alpha\neq2$, the divergence $D^\Holder_\alpha$
is not even proper for normalized (probability) densities, not to mention general
unnormalized (positive) densities. 
Indeed, when $p(x)=q(x)$, we have:
\begin{equation}
D^\Holder_{\alpha}(p(x):p(x))
=-\log\left(\frac{\int p(x)^2 \dx}{\left(\int p(x)^\alpha\dx\right)^{\frac{1}{\alpha}} \left(\int p(x)^{\frac{\alpha}{\alpha-1}}\dx\right)^{\frac{\alpha-1}{\alpha}}}\right) \not = 0 \mbox{ when $\alpha\not =2$}.
\end{equation}
Let us consider the general case. For unnormalized positive distributions $\tp(x)$ and $\tq(x)$ (the tilde notation stems from the notation of homogeneous coordinates in projective geometry), the inequality becomes an equality when:
$\tp(x)^\alpha \propto \tq(x)^{\beta}$, i.e., $p(x)^\alpha\propto{q(x)}^\beta$, or
$q(x)\propto{p(x)}^{\alpha/\baralpha}=p(x)^{\alpha-1}$.  
We can check that $D^\Holder_\alpha(p(x):\lambda p(x)^{\alpha-1})=0$ for any $\lambda>0$:
\begin{equation}
-\log\left(\frac{\int p(x)\lambda p(x)^{\alpha-1}\dx}
{\left(\int p(x)^\alpha\dx\right)^{\frac{1}{\alpha}}
\left(\int \lambda^\beta p(x)^{(\alpha-1)\beta}\dx\right)^{\frac{1}{\beta}}}\right)
=-\log \left(\frac{\int p(x)^\alpha\dx}{\left(\int p(x)^\alpha\dx\right)^{\frac{1}{\alpha}} \left(\int  p(x)^\alpha\dx\right)^{\frac{1}{\beta}}}\right)=0,
\end{equation}
since $(\alpha-1)\beta=(\alpha-1)\baralpha=(\alpha-1)\frac{\alpha}{\alpha-1}=\alpha$.

For $\alpha=2$, we find indeed that $D^\Holder_2(p(x):\lambda p(x))=\CS(p(x):p(x))=0$ for any
$\lambda\neq0$.   

\begin{fact}[HPD is improper]
The H\"older pseudo-divergences are improper statistical distances.
\end{fact}

\subsection{Reference duality}
In general, H\"older divergences are asymmetric when $\alpha\neq\beta$ ($\not =2$)  but enjoy the following {\em reference duality}~\cite{referenceduality-2015}:
\begin{equation}
D^\Holder_\alpha(p(x):q(x))=D^\Holder_{\beta}(q(x):p(x))=D^\Holder_{\frac{\alpha}{\alpha-1}}(q(x):p(x)).
\end{equation}

\begin{fact}[Reference duality HPD]
The H\"older pseudo-divergences satisfy the reference duality $\alpha\leftrightarrow\beta=\frac{\alpha}{\alpha-1}$:
$D^\Holder_\alpha(p(x):q(x))=D^\Holder_{\beta}(q(x):p(x))=D^\Holder_{\frac{\alpha}{\alpha-1}}(q(x):p(x))$.
\end{fact}

An arithmetic symmetrization of the HPD yields a symmetric HPD $S^\Holder_\alpha$, given by:
\begin{eqnarray}
S^\Holder_\alpha(p(x):q(x)) &=& S^\Holder_\alpha(q(x):p(x))
=\frac{D^\Holder_\alpha(p(x):q(x))+D^\Holder_{\bar\alpha}(q(x):p(x))}{2},\nonumber\\
&=& -\log\left(
\frac{ \int p(x)q(x)\dx}{\sqrt{
\left( \int p(x)^\alpha\dx\right)^{\frac{1}{\alpha}}
 \left(\int p(x)^\baralpha\dx\right)^{\frac{1}{\baralpha}}
 \left(\int q(x)^\alpha\dx\right)^{\frac{1}{\alpha}}  \left(\int q(x)^\baralpha\dx\right)^{\frac{1}{\baralpha}}
}}
\right).
\end{eqnarray}

\subsection{HPD is a projective divergence}
In the above definition, densities $p(x)$ and $q(x)$ can either be positive or normalized probability distributions.
Let $\tp(x)$ and $\tq(x)$ denote positive (not necessarily normalized) measures, and $w(\tp)=\int_\calX \tp(x)\dx$
the {\em overall mass} so that $p(x)=\frac{\tp(x)}{w(\tp)}$ is the corresponding normalized probability measure.
Then we check that HPD is a {\em projective divergence}~\cite{GammaDiv-2008} since:
\begin{equation}
D^\Holder_\alpha(\tp(x):\tq(x)) = D^\Holder_\alpha(p(x):q(x)),
\end{equation}
or in general:
\begin{equation}
D^\Holder_\alpha(\lambda p(x):\lambda' q(x))=D^\Holder_\alpha(p(x):q(x))
\end{equation} 
for all prescribed constants $\lambda,\lambda'>0$.
Projective divergences may also be called ``{\em angular divergences}'' or ``{\em cosine divergences}'' since they do not depend on the total mass of the measure densities.

\begin{fact}[HPD is projective]
The H\"older pseudo-divergences are projective distances.
\end{fact}

\subsection{Escort distributions and skew Bhattacharyya divergences}\label{sec:escort}

Let us define with respect to the probability measures $p(x)\in L^{\frac{1}{\alpha}}(\calX,\mu)$ and $q(x)\in L^{\frac{1}{\beta}}(\calX,\mu)$ the
following {\em escort probability distributions}~\cite{IG-2016}:
\begin{equation}
p_\alpha^{E}(x)=\frac{p(x)^{\frac{1}{\alpha}}}{\int p(x)^{\frac{1}{\alpha}}\dx},
\end{equation} and 
\begin{equation}
q_\beta^E(x)=\frac{q(x)^{\frac{1}{\beta}}}{\int q(x)^{\frac{1}{\beta}} \dx}.
\end{equation}

Since HPD is a projective divergence, we compute 
with respect to the conjugate exponents $\alpha$ and $\beta$
the {\em H\"older Escort Divergence} (HED):
\begin{equation}
D^\HE_\alpha(p(x):q(x))
=D^\Holder_{\alpha}(p_\alpha^E(x):q_\beta^E(x))
=-\log \int_\calX p(x)^{1/\alpha}q(x)^{1-1/\alpha}\dx=B_{1/\alpha}(p(x):q(x)),
\end{equation}
which turns out to be the familiar {\em skew Bhattacharyya divergence} $B_{1/\alpha}(p(x):q(x))$, see~\cite{BR-2011}.

\begin{fact}[HED as a skew Bhattacharyya divergence]
The H\"older escort divergence amounts to a skew Bhattacharyya divergence:
$D^\HE_\alpha(p(x):q(x)) = B_{1/\alpha}(p(x):q(x))$ for any $\alpha>0$.
\end{fact}

In particular, the {\em Cauchy-Schwarz escort divergence} $\CS^\HE(p(x):q(x))$ amounts to the Bhattacharyya distance~\cite{Bhatt-1943} $B(p(x):q(x))=-\log\sqrt{\int_\calX p(x)q(x)\dx}$: 
\begin{equation}
\CS^\HE(p(x):q(x))=D^\HE_2(p(x):q(x))=D^H_2(p^E_2(x):q^E_2(x)) =B_{1/2}(p(x):q(x))=B(p(x):q(x)).
\end{equation}

Observe that the Cauchy-Schwarz escort distributions are the square root density representations~\cite{sqrRep-2007} of distributions.

\section{Proper H\"older divergence}\label{sec:hd}

\subsection{Definition}

Let $p(x)$ and $q(x)$ be positive measures in $L^\gamma(\calX,\mu)$ for a prescribed scalar value $\gamma>0$.
Plugging the positive measures $p(x)^{\gamma/\alpha}$ and $q(x)^{\gamma/\beta}$ into the definition of
HPD $D^\Holder_\alpha$, we get the following definition:

\begin{definition}[Proper H\"older Divergence, HD]
For conjugate exponents $\alpha, \beta>0$
and $\gamma>0$, the proper  H\"older divergence between two densities $p(x)$ and $q(x)$ is defined by:
\begin{equation}
D^\Holder_{\alpha,\gamma}(p(x):q(x))
=D^\Holder_{\alpha}(p(x)^{\gamma/\alpha}:q(x)^{\gamma/\beta})
= - \log\left(\frac{\int_\calX p(x)^{\gamma/\alpha} q(x)^{\gamma/\beta} \dx}
{(\int_\calX p(x)^\gamma \dx)^{1/\alpha} (\int_\calX q(x)^\gamma \dx)^{1/\beta}}
\right).
\end{equation}
\end{definition}

By definition, $D^\Holder_{\alpha,\gamma}(p:q)$ is a \emph{two-parameter family} of
dissimilarity statistical measures.
Following H\"older's inequality, we can check that $D^\Holder_{\alpha,\gamma}(p(x):q(x))\ge0$ and
$D^\Holder_{\alpha,\gamma}(p(x):q(x))=0$ iff 
$p(x)^\gamma\propto{q}(x)^\gamma$, i.e. $p(x)\propto{q}(x)$ (see Appendix \ref{sec:hdproof}).
If $p(x)$ and $q(x)$ belong to the statistical probability manifold, then
$D^\Holder_{\alpha,\gamma}(p(x):q(x))=0$ iff $p(x)=q(x)$ almost everywhere.
This says that HD is a proper divergence for probability measures,
and it becomes a pseudo-divergence for positive measures.  Note that we have abused
the notation $D^\Holder$ to denote both the H\"older   pseudo-divergence (with
one subscript) and the H\"older   divergence (with two subscripts).

Similar to HPD, HD is asymmetric when $\alpha\neq\beta$ with the following
reference duality:
\begin{equation}
D^\Holder_{\alpha,\gamma}(p(x):q(x)) = D^\Holder_{\baralpha,\gamma}(q(x),p(x)).
\end{equation}

HD can be symmetrized as:
\begin{equation}
S^\Holder_{\alpha,\gamma}(p:q) =
\frac{D^\Holder_{\alpha,\gamma}(p:q)+D^\Holder_{\alpha,\gamma}(q:p)}{2}
=
-\log\sqrt{\frac{\int_\calX p(x)^{\gamma/\alpha} q(x)^{\gamma/\beta} \dx
\int_\calX p(x)^{\gamma/\beta} q(x)^{\gamma/\alpha} \dx}
{\int_\calX p(x)^\gamma \dx \int_\calX q(x)^\gamma \dx }}.
\end{equation}
Furthermore, one can easily check that HD is a projective divergence. 

For conjugate exponents $\alpha,\beta>0$ and $\gamma>0$, we rewrite the definition of HD as:
\begin{align*}
D^\Holder_{\alpha,\gamma}(p(x):q(x))
&= -\log
\int_\calX
\left(\frac{p(x)^{\gamma}}{\int_\calX p(x)^\gamma \dx}\right)^{1/\alpha}
\left(\frac{q(x)^{\gamma}}{\int_\calX q(x)^\gamma \dx}\right)^{1/\beta}
\dx,\\
&= -\log \left(p_{1/\gamma}^E(x)\right)^{1/\alpha} \left(q_{1/\gamma}^E(x)\right)^{1/\beta} \dx
= B_{\frac{1}{\alpha}}(p_{1/\gamma}^E(x):q_{1/\gamma}^E(x)).
\end{align*}
Therefore HD  can be reinterpreted as the skew Bhattacharyya divergence~\cite{BR-2011} between the escort distributions.
In particular, when $\gamma=1$, we get:
\begin{equation}
D^\Holder_{\alpha,1}(p(x):q(x)) = -\log\left(\int_\calX p(x)^{1/\alpha} q(x)^{1/\beta} \dx\right)
= B_{\frac{1}{\alpha}}(p(x):q(x)).
\end{equation}

\begin{fact}
The two-parametric family of statistical H\"older divergence $D^\Holder_{\alpha,\gamma}$
passes through the one-parametric family of skew Bhattacharyya divergences when $\gamma=1$.
\end{fact}
 
\subsection{Special case: The Cauchy-Schwarz divergence}

We consider the intersection of the uni-parametric class of H\"older pseudo-divergences (HPD) with the bi-parametric class of proper  H\"older divergences (HD):
That is, the class of divergences which belong to both HPD and HD. Then we must have
$\gamma/\alpha=\gamma/\beta=1$. Since $1/\alpha+1/\beta=1$,
we get $\alpha=\beta=\gamma=2$. 
Therefore the Cauchy-Schwarz (CS) divergence is the {\em unique} divergence belonging to both HPD and HD classes:
\begin{equation}
D^\Holder_{2,2}(p(x):q(x))=D^\Holder_{2}(p(x):q(x))=\CS(p(x):q(x)).
\end{equation}
In fact, the CS divergence is the intersection of the four classes HPD, symmetric HPD, HD, and symmetric HD.
Figure~\ref{fig:set} displays a diagram of those divergence classes with their inclusion relationships.

\begin{figure}[t]
\centering
\includegraphics[width=.95\textwidth]{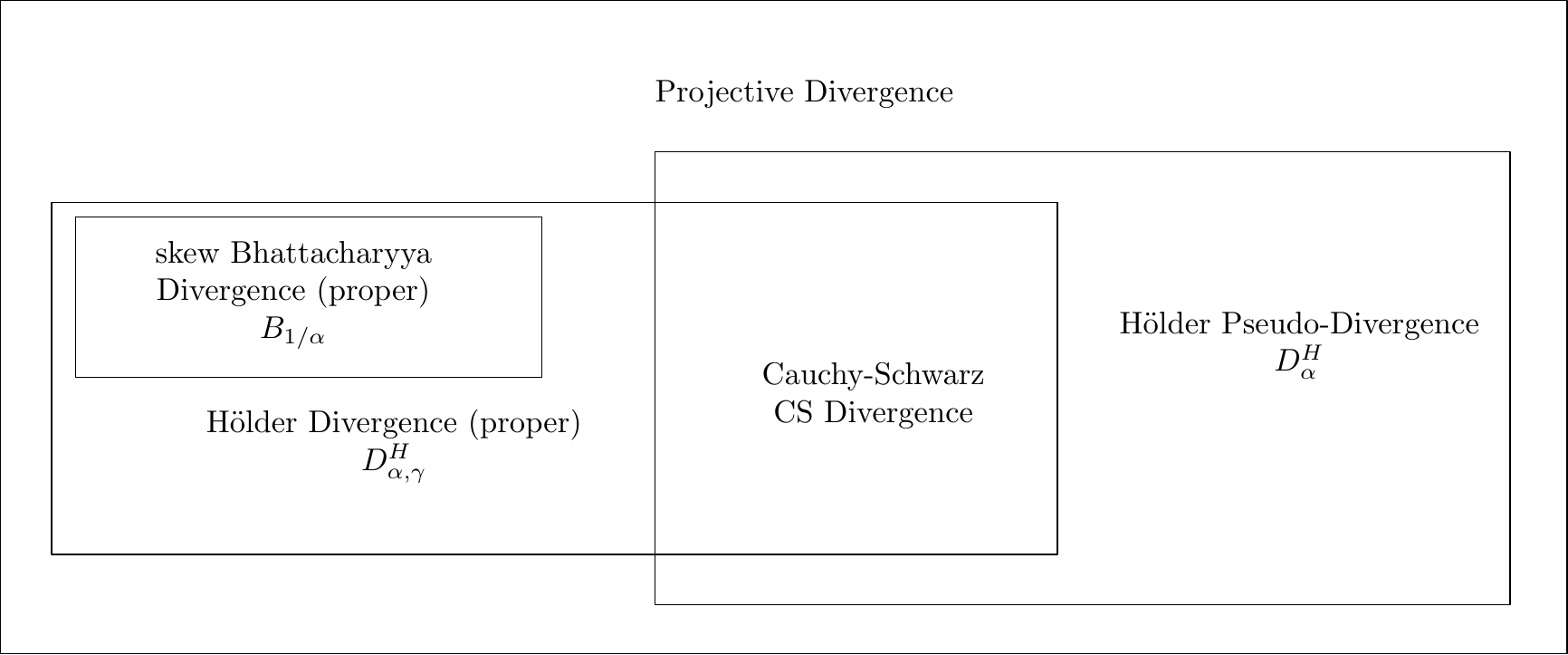}

\caption{H\"older proper divergence (bi-parametric) and H\"older improper pseudo-divergence (uni-parametric) intersect at the unique non-parametric Cauchy-Schwarz divergence.  By using escort distributions, H\"older divergences encapsulates the skew Bhattacharyya distances.}\label{fig:set}
\end{figure}

As stated earlier, notice that the Cauchy-Schwarz inequality:
\begin{equation}
\int p(x)q(x)\dx \leq \sqrt{\left(\int p(x)^2 \dx\right)\left(\int p(x)^2 \dx\right)},
\end{equation}
is not proper as it is an equality when $p(x)$ and $q(x)$ are {\em linearly dependent} (i.e., $p(x)=\lambda q(x)$ for $\lambda>0$).
The arguments of the CS divergence are square-integrable real-valued density functions $p(x)$ and $q(x)$.
Thus the Cauchy-Schwarz divergence is not proper for positive measures but is proper for normalized probability distributions since $\int p(x)\dx=\int \lambda q(x)\dx=1$ implies that $\lambda=1$.

\subsection{Limit cases of H\"older divergences and statistical estimation}

Let us define the inner product of unnormalized densities as:
\begin{equation}
\inner{\tp(x)}{\tq(x)}=\int_\calX \tp(x)\tq(x)\dx
\end{equation} 
(for $L^2(\calX,\mu)$ integrable functions),
and define the $L_\alpha$ norm of densities as
$\Vert\tp(x)\Vert_\alpha = (\int_\calX \tp(x)^{\alpha} \dx)^{1/\alpha}$ for $\alpha\geq 1$.
Then the CS divergence can be concisely written  as:
\begin{equation}
\CS(\tp(x),\tq(x))=-\log \frac{\inner{\tp(x)}{\tq(x)}}{\Vert{\tp(x)}\Vert_2 \Vert{\tq(x)}\Vert_2},
\end{equation}
and the H\"older pseudo-divergence writes as:
\begin{equation}
D^\Holder_\alpha(\tp(x),\tq(x))
= -\log \frac{\inner{\tp(x)}{\tq(x)}}{\Vert{\tp(x)}\Vert_\alpha\Vert{\tq(x)}\Vert_\baralpha}.
\end{equation}

When $\alpha\rightarrow 1^{+}$, we have $\baralpha=\alpha/(\alpha-1) \rightarrow +\infty$. 
Then it comes that:
\begin{equation}
\lim_{\alpha\to1^+} D^\Holder_\alpha(\tp(x),\tq(x))
= -\log\frac{\inner{\tp(x)}{\tq(x)}}{\Vert\tp(x)\Vert_1 \Vert\tq(x)\Vert_\infty}
= -\log{\inner{\tp(x)}{\tq(x)}} +\log\int_\calX\tp(x)\dx + \log\max_{x\in\calX} \tq(x).
\end{equation}
When $\alpha\to+\infty$ and $\baralpha\to1^+$, we have:
\begin{equation}
\lim_{\alpha\to+\infty} D^\Holder_\alpha(\tp(x),\tq(x))
= -\log\frac{\inner{\tp(x)}{\tq(x)}}{\Vert\tp(x)\Vert_\infty \Vert\tq(x)\Vert_1}
= -\log{\inner{\tp(x)}{\tq(x)}} + \log\max_{x\in\calX}\tp(x) + \log\int_\calX\tq(x)\dx.
\end{equation}

Now consider a pair of probability densities $p(x)$ and $q(x)$. We have:
\begin{align}
\lim_{\alpha\to1^+} D^\Holder_\alpha(p(x),q(x)) &= -\log\inner{p(x)}{q(x)} + \max_{x\in\calX}\log{q}(x),\nonumber\\
\lim_{\alpha\to+\infty} D^\Holder_\alpha(p,q) &= -\log\inner{p(x)}{q(x)} + \max_{x\in\calX}\log{p}(x),\nonumber\\
D^\Holder_2(p,q) &= -\log\inner{p(x)}{q(x)} + \log\Vert{p(x)}\Vert_2 + \log\Vert{q(x)}\Vert_2.
\end{align}
In an estimation scenario, $p(x)$ is fixed and $q(x\,\vert\,\theta)=q_\theta(x)$
is free along a parametric manifold $\calM$, then minimizing H\"older divergence reduces to:
\begin{align}
\argmin_{\theta\in\calM}\lim_{\alpha\to1^+} D^\Holder_\alpha(p(x),q_\theta(x))
&= \argmin_{\theta\in\calM}\bigg( -\log\inner{p(x)}{q_\theta(x)} + \max_{x\in\calX}\log{q}_\theta(x) \bigg),\nonumber\\
\argmin_{\theta\in\calM}\lim_{\alpha\to+\infty} D^\Holder_\alpha(p(x),q(x))
&= \argmin_{\theta\in\calM}\bigg( -\log\inner{p(x)}{q_\theta(x)} \bigg),\nonumber\\
\argmin_{\theta\in\calM} D^\Holder_2(p(x),q(x))
&= \argmin_{\theta\in\calM}\bigg( -\log\inner{p(x)}{q_\theta(x)} + \log\Vert{q_\theta(x)}\Vert_2\bigg).
\end{align}
Therefore when $\theta$ varies from 1 to $+\infty$, only the regularizer
in the minimization problem   changes.  In any case, H\"older divergence
always has the term $-\log\inner{p(x)}{q(x)}$, which shares a similar form with the Bhattacharyya distance~\cite{Bhatt-1943}:
\begin{equation}
B(p(x):q(x))= -\log \int_\calX \sqrt{p(x)q(x)} \dx = -\log\inner{\sqrt{p(x)}}{\sqrt{q(x)}}.
\end{equation}
HPD between $\tp(x)$ and $\tq(x)$ is also closely related to their cosine similarity
$\frac{\inner{\tp(x)}{\tq(x)}}{\Vert\tp(x)\Vert_2\Vert\tq(x)\Vert_2}$. When $\alpha=2$,
HD is exactly the cosine similarity after a non-linear transformation.

\section{Closed-form expressions of HPD and HD for conic and affine exponential families}\label{sec:HCEF}

We report closed-form formulas for the HPD and HD between two distributions belonging to the same exponential family provided that the natural parameter space is a cone or affine. A cone $\Omega$ is a convex domain such that for $P, Q\in \Omega$ and any $\lambda>0$, we have $P+\lambda Q\in\Omega$.
For example, the set of positive measures absolutely continuous with a base measure $\mu$ is a cone.
Recall that an exponential family~\cite{EF-2009} has a  density function $p(x;\theta)$ that we be written canonically as:

\begin{equation}
p(x;\theta)=\exp\left(\inner{t(x)}{\theta}-F(\theta)+k(x)\right).
\end{equation} 

In this work, we consider the auxiliary carrier measure term $k(x)=0$. The base measure is either the Lebesgue measure $\mu$ or the counting measure $\mu_C$.
A Conic or Affine Exponential Family (CAEF)  is an exponential family with the natural parameter space $\Theta$ a cone or affine.
The log-normalizer $F(\theta)$ is a strictly convex function also called cumulant generating function~\cite{IG-2016}.

\begin{lemma}[HPD and HD for CAEFs]\label{lemma:holderclosedform}
For distributions $p(x;\theta_p)$ and $p(x;\theta_q)$ belonging to the same exponential family
with conic or affine natural parameter space~\cite{DivMix-2012}, both the HPD and HD are available
in closed-form:
\begin{align}
D^\Holder_{\alpha}(p:q)
&= \frac{1}{\alpha}F(\alpha\theta_p)+ \frac{1}{\beta}F(\beta\theta_q)-F(\theta_p+\theta_q),\\
D^\Holder_{\alpha,\gamma}(p:q)
&= \frac{1}{\alpha}F(\gamma\theta_p) + \frac{1}{\beta}F(\gamma\theta_q) 
-F\left(\frac{\gamma}{\alpha}\theta_p + \frac{\gamma}{\beta}\theta_q\right).
\end{align}
\end{lemma}

\begin{proof}
Consider $k(x)=0$ and a conic or affine natural space $\Theta$ (see~\cite{DivMix-2012}), then for all $a,b>0$, we have:
\begin{equation}
\left(\int p(x)^a \dx\right)^{\frac{1}{b}} = \exp\left(\frac{1}{b}F(a\theta_p)-\frac{a}{b}F(\theta_p)\right),
\end{equation}
since $a\theta_p\in\Theta$. Indeed, we have:
\begin{align*}
\left(\int p(x)^a \dx\right)^{1/b}
&= \left( \int \exp\left(\inner{a\theta}{t(x)}-aF(\theta)\right)\dx \right)^{1/b}\\
&= \left( \int \exp\left(\inner{a\theta}{t(x)}-F(a\theta)+ F(a\theta)-aF(\theta)\right) \dx \right)^{1/b} \\
&= \exp\left(\frac{1}{b}F(a\theta)-\frac{a}{b}F(\theta)\right)
\left(\underbrace{\int\exp\left(\inner{a\theta}{t(x)}-F(a\theta)\right)\dx}_{=1}\right)^{1/b}.
\end{align*}

Similarly, we have for all ${a},{b}>0$ (details omitted),
\begin{equation}
\int p(x)^a q(x)^b\dx = \exp(F(a\theta_p+b\theta_q)-aF(\theta_p)-bF(\theta_q)),
\end{equation}
since $a\theta_p+b\theta_q\in\Theta$. 
Therefore, we get:
\begin{align}
D^\Holder_\alpha(p(x):q(x))
&= -\log \frac{\int p(x)q(x)\dx}{\left(\int p(x)^\alpha\dx\right)^{\frac{1}{\alpha}} \left(\int q(x)^\beta\dx\right)^{\frac{1}{\beta}} }\\
&= -F(\theta_p+\theta_q)+F(\theta_p)+F(\theta_q)
+ \frac{1}{\alpha}F(\alpha\theta_p)- F(\theta_p)
+  \frac{1}{\beta}F(\beta\theta_q)- F(\theta_q)\\
&=   \frac{1}{\alpha}F(\alpha\theta_p)+ \frac{1}{\beta}F(\beta\theta_q)-F(\theta_p+\theta_q) \geq 0,\\
D^\Holder_{\alpha,\gamma}(p(x):q(x))
&= -\log \frac{\int p(x)^{\gamma/\alpha}q(x)^{\gamma/\beta}\dx}
{\left(\int p(x)^\gamma\dx\right)^{\frac{1}{\alpha}} \left(\int q(x)^\gamma\dx\right)^{\frac{1}{\beta}} }\\
&=-F\left(\frac{\gamma}{\alpha}\theta_p + \frac{\gamma}{\beta}\theta_q\right)
+\frac{\gamma}{\alpha}F(\theta_p) + \frac{\gamma}{\beta}F(\theta_q)
+\frac{1}{\alpha}F(\gamma\theta_p) - \frac{\gamma}{\alpha}F(\theta_p)
+\frac{1}{\beta}F(\gamma\theta_q) - \frac{\gamma}{\beta}F(\theta_q)\\
&=
\frac{1}{\alpha}F(\gamma\theta_p) + \frac{1}{\beta}F(\gamma\theta_q) 
-F\left(\frac{\gamma}{\alpha}\theta_p + \frac{\gamma}{\beta}\theta_q\right)\geq 0.
\end{align}

When $1>\alpha>0$, we have $\beta=\frac{\alpha}{\alpha-1}<0$. To get similar
results for the reverse H\"older divergence, we need the natural parameter
space to be affine (eg., isotropic Gaussians or multinomials, see~\cite{fdivChi-2014}).
\end{proof}

In particular, if $p(x)$ and $q(x)$ belong to the  same exponential family so
that $p(x)=\exp(\inner{\theta_p}{t(x)} - F(\theta_p))$ and 
$q(x)=\exp(\inner{\theta_q}{t(x)}-F(\theta_q))$, one can easily check that $D^\Holder_\alpha(p(x;\theta_p):p(x;\theta_q))=0$ iff $\theta_q=(\alpha-1)\theta_p$.
For HD, we can check $D^\Holder_{\alpha,\gamma}(p(x):p(x))=0$ is proper since $\frac{1}{\alpha}+\frac{1}{\beta}=1$.

The following result is straightforward from Lemma~\ref{lemma:holderclosedform}.
\begin{lemma}[Symmetric HPD and HD for CAEFs]
For distributions $p(x;\theta_p)$ and $p(x;\theta_q)$ belonging to the same exponential family
with conic or affine natural parameter space~\cite{DivMix-2012}, the symmetric HPD and HD are available
in closed-form:
\begin{align}
S^\Holder_{\alpha}(p(x):q(x))
&=
\frac{1}{2}\left[
\frac{1}{\alpha}F(\alpha\theta_p) +\frac{1}{\beta}F(\beta\theta_p)
+\frac{1}{\alpha}F(\alpha\theta_q) + \frac{1}{\beta}F(\beta\theta_q)
\right] -F(\theta_p+\theta_q); \\
S^\Holder_{\alpha,\gamma}(p(x):q(x))
&= \frac{1}{2}
\left[ F(\gamma\theta_p) + F(\gamma\theta_q) 
-F\left(\frac{\gamma}{\alpha}\theta_p + \frac{\gamma}{\beta}\theta_q\right)
-F\left(\frac{\gamma}{\beta}\theta_p + \frac{\gamma}{\alpha}\theta_q\right)\right].
\end{align}
\end{lemma}

\begin{remark}
By reference duality,
\begin{align*}
S^\Holder_{\alpha}(p(x):q(x))        &= S^\Holder_{\baralpha}(p(x):q(x));\\
S^\Holder_{\alpha,\gamma}(p(x):q(x)) &= S^\Holder_{\baralpha,\gamma}(p(x):q(x)).
\end{align*}
\end{remark}

Note that the H\"older score-induced divergence~\cite{scaleinvariantDiv-2014} does {\em not} admit in general closed-form formulas  for exponential families since it relies on a {\em  function} $\phi(\cdot)$ (see Definition~4 of~\cite{scaleinvariantDiv-2014}).

Note that CAEF convex log-normalizers satisfy:
\begin{equation}
\frac{1}{\alpha}F(\alpha\theta_p)+ \frac{1}{\beta}F(\beta\theta_q) \geq F(\theta_p+\theta_q).
\end{equation}
A necessary condition is that $F(\lambda\theta)\geq \lambda F(\theta)$ for $\lambda>0$ (take $\theta_p=\theta$, $\theta_q=0$ and $F(0)=0$  in the above equality).

The escort distribution for an exponential family is given by:
\begin{equation}
p_\alpha^E(x;\theta) =
e^{\frac{F(\theta)}{\alpha}-F(\frac{\theta}{\alpha})}
p(x;\theta)^{\frac{1}{\alpha}}.
\end{equation}

The H\"older equality holds when $p(x)^\alpha \propto q(x)^\beta$ or $p(x)^\alpha q(x)^{-\beta}  \propto  1 $.
For exponential families, this condition is satisfied when $\alpha\theta_p - \beta\theta_q\in\Theta$.
That is, we need to have:

\begin{equation}
\alpha \left(\theta_p -  \frac{1}{\alpha-1}\theta_q\right) \in\Theta.
\end{equation}

Thus we may choose small enough $\alpha=1+\epsilon>1$ so that the condition is not satisfied for fixed $\theta_p$ and $\theta_q$ for many exponential distributions.
Since multinomials have affine natural space~\cite{fdivChi-2014}, this condition is always met, but not for non-affine natural parameter spaces like normal distributions.

Notice the following fact:
\begin{fact}[Density of a CAEF in $L^\gamma(\calX,\mu)$]
The density of exponential families with conic or affine natural parameter space belongs to $L^\gamma(\calX,\mu)$ for any $\gamma>0$.
\end{fact}

\begin{proof}
We have $\int_\calX (\exp(\inner{\theta}{t(x)}-F(\theta)))^\gamma \dmu(x)= e^{F(\gamma\theta)-\gamma F(\theta)}<\infty$ for any $\gamma>0$ provided
that $\gamma\theta$ belongs to the natural parameter space. When $\Theta$ is a cone or affine, the condition is satisfied.
\end{proof}

Let $\tp(x;\theta)=\exp\left(\inner{t(x)}{\theta)}\right)$ denote the unnormalized positive exponential family density and 
$p(x;\theta)=\frac{\tp(x;\theta)}{Z(\theta)}$ the normalized density with $Z(\theta)=\exp(F(\theta))$ the partition function.
Although HD is a projective divergence since we have $D^\Holder_\alpha(p(x;\theta_1):p(x;\theta_2)) =D^\Holder_\alpha(\tp(x;\theta_1):\tp(x;\theta_2))$,
observe that the HD value {\em depends} on the log-normalizer $F(\theta)$ (since the HD is an integral on the support, see also~\cite{PM-ProjDiv-2016} for a similar argument with the $\gamma$-divergence~\cite{GammaDiv-2008}). 

In practice, even when the log-normalizer is computationally intractable, we may still estimate the HD by Monte-Carlo techniques:
Indeed, we can sample a distribution $\tp(x)$ either by rejection sampling~\cite{PM-ProjDiv-2016} or by the Markov Chain Monte-Carlo (MCMC) Metropolis-Hasting technique: It just requires to be able to sample a proposal distribution that has the same support.

We shall now instantiate the HPD and HD formulas for several exponential families with conic or affine natural parameter spaces.

\subsection{Case study: Categorical distributions}

Let $p=(p_0,\cdots,p_m)$
and $q=(q_0,\cdots,q_m)$ be two categorical distributions in the 
$m$-dimensional probability simplex $\Delta^m$. We rewrite $p$
in the canonical form of exponential families~\cite{EF-2009} as:
\begin{equation}\label{eq:pcat}
p_i = \exp\left((\theta_p)_i - \log\left(1+\sum_{i=1}^m\exp(\theta_p)_i\right)\right),
\quad\forall{i}\in\{1,\cdots,m\},
\end{equation}
with the redundant parameter:
\begin{equation}
p_0 = 1-\sum_{i=1}^mp_i = 
\frac{1}{1+\sum_{i=1}^m\exp(\theta_p)_i}.
\end{equation}
From Eq. \ref{eq:pcat}, the convex cumulant generating function has the form $F(\theta)=\log\left(1+\sum_{i=1}^m\exp(\theta_p)_i\right)$.
The inverse transformation from $p$ to $\theta$ is therefore given by:
\begin{equation}
\theta_i = \log\left(\frac{p_i}{p_0}\right),
\quad\forall{i}\in\{1,\cdots,m\}.
\end{equation}
The natural parameter space $\Theta$ is affine (hence conic), and by applying Lemma~\ref{lemma:holderclosedform}, we get the following closed-form formula:
\begin{align}
D^\Holder_\alpha(p:q) =&
  \frac{1}{\alpha}\log\left(1+\sum_{i=1}^m\exp(\alpha(\theta_p)_i)\right)
+ \frac{1}{\beta}\log\left(1+\sum_{i=1}^m\exp(\beta(\theta_q)_i)\right)\nonumber\\
&
- \log\left(1+\sum_{i=1}^m\exp((\theta_p)_i+(\theta_q)_i)\right),\\
D^\Holder_{\alpha,\gamma}(p:q) =&
\frac{1}{\alpha}\log\left(1+\sum_{i=1}^m\exp(\gamma(\theta_p)_i)\right)
+ \frac{1}{\beta}\log\left(1+\sum_{i=1}^m\exp(\gamma(\theta_q)_i)\right)\nonumber\\
&
- \log\left(1+\sum_{i=1}^m\exp\left(\frac{\gamma}{\alpha}(\theta_p)_i+\frac{\gamma}{\beta}(\theta_q)_i\right)\right).
\end{align}

\begin{figure}[t]
 \centering
  \begin{subfigure}[b]{\textwidth}
    \centering\includegraphics[width=\textwidth]{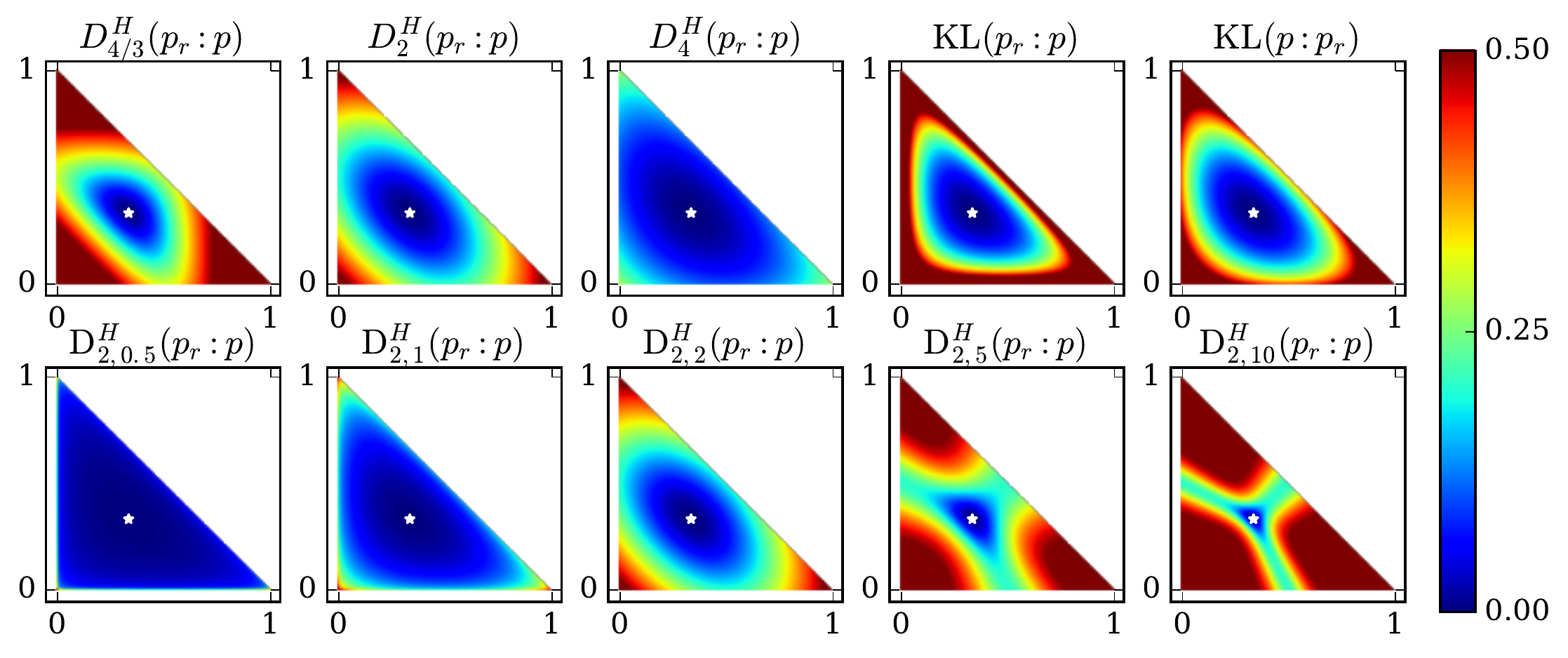}
    \caption{Reference categorical distribution $p_r=(1/3,1/3,1/3)$\label{subfig:cat1}}
  \end{subfigure}
  \begin{subfigure}[b]{\textwidth}
    \centering\includegraphics[width=\textwidth]{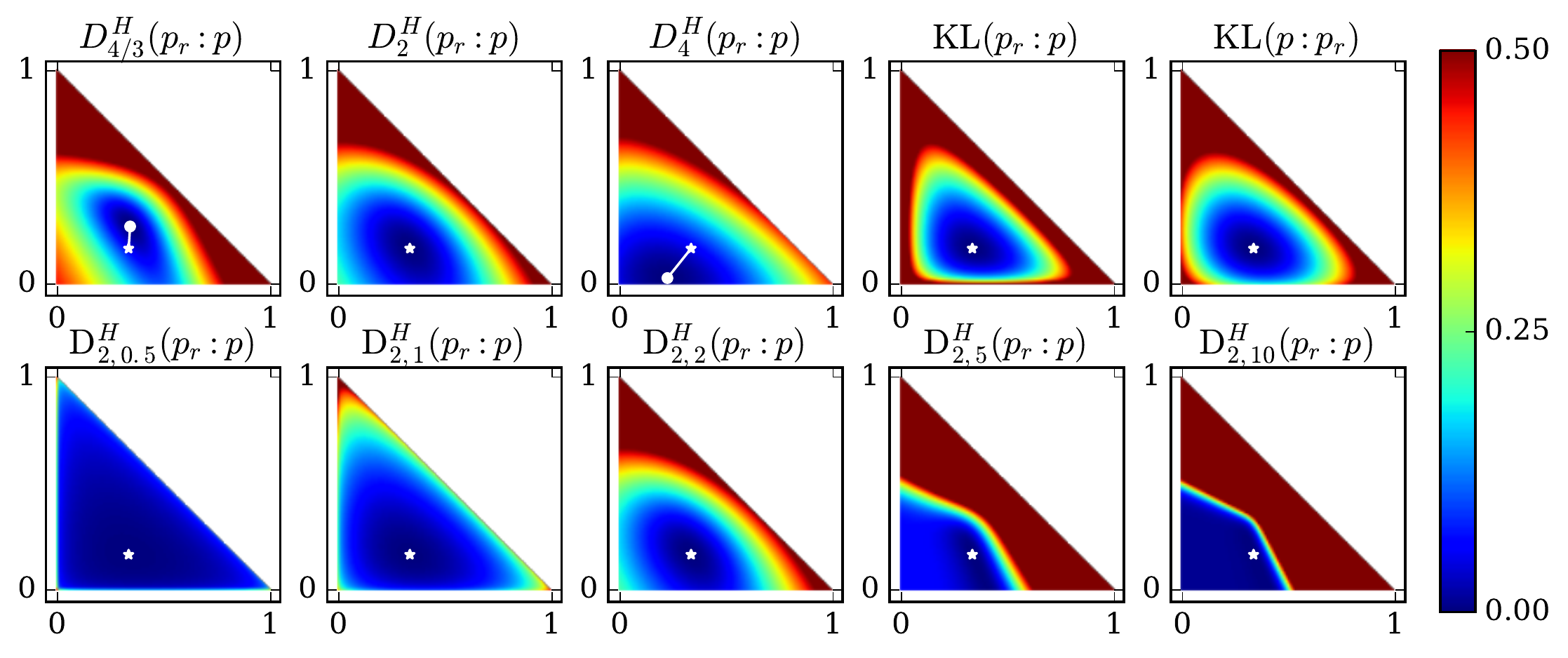}
    \caption{Reference categorical distribution $p_r=(1/2,1/3,1/6)$\label{subfig:cat2}}
  \end{subfigure}
\caption{$D^\Holder_\alpha(p_r:p)$ for $\alpha\in\{4,2,4/3\}$,
and $D^\Holder_{2,\gamma}(p_r:p)$ for $\gamma\in\{0.5,1,2,5,10\}$,
compared to KL divergence.
The reference distribution $p_r$ is presented as ``$\star$''. The minimizer
of $D^\Holder_\alpha(p_0:p)$, if different from $p_0$, is presented as
``$\bullet$''. Notice that $D^\Holder_{2,2}=D^\Holder_2$.\label{fig:cat}}
\end{figure}

To get some intuitions,  Fig.~\ref{fig:cat} shows the H\"older divergence from any
 categorical distribution $(p_0,p_1,p_2)$ to a given reference distribution $p_r$ in the 2D probability simplex $\Delta^2$.
A main observation is that the Kullback-Leibler (KL) divergence exhibits a barrier near the boundary $\partial\Delta^2$ with
large values. This is {\em not} the case for H\"older divergences: $D^\Holder_\alpha(p:p_r)$
does not have a sharp increase near the boundary (although it still penalizes the corners
of $\Delta^2$). For example, let $p=(0,1/2,1/2)$, $p_r=(1/3,1/3,1/3)$, then 
$\KL(p_r:p)\to\infty$ but $D^\Holder_2(p_r:p)=\sqrt{2/3}$.
Another observation is that the minimum $D(p:p_r)$ can be reached at some point $p\neq{p_r}$ (see for example $D^\Holder_4(p:p_r)$ in Fig.~\ref{subfig:cat2}, the bluest area corresponding to the minimum of $D(p:p_r)$ is not in the same location as the reference point).

Consider a HPD ball of center $c$ and prescribed radius $r$ wrt the HPD.
Since $p(x)^{\alpha-1}$  for $\alpha\not =2$ does not belong to the probability manifold but to the positive measure manifold and since the distance is projective, 
we deduce that the displaced ball center $c'$ of a ball $c$ lying the probability manifold can be computed as the
intersection of the ray $0p^{\alpha-1}$ with the probability manifold.
For the discrete probability simplex $\Delta$, since we have $\lambda \sum_{x\in\calX} p(x)^{\alpha-1}=1$, we deduce that the displaced ball center is at:

\begin{equation}
c'=\frac{c}{\sum_{x\in\calX} p(x)^{\alpha-1}}
\end{equation}
This center is  displayed as ``$\bullet$'' in Figure~\ref{fig:cat}.

In general, the HPD bisector~\cite{skewJB} between two distributions belonging to the same CAEF is defined by:
\begin{equation}
\frac{1}{\alpha} (F(\alpha\theta_1) - F(\alpha\theta_2)) = F(\theta_2+\theta) - F(\theta_1+\theta).
\end{equation}

\subsection{Case study: Bernoulli distribution}
Bernoulli distribution is just a special case of the category distribution
when the number of categories is $2$  (i.e., $m=1$).
To be consistent with the previous section, we rewrite a Bernoulli distribution $p=(p_0,p_1)$
in the canonical form:
\begin{equation} 
p_1 = \exp\left(\theta_p - \log\left(1+\exp(\theta_p)\right)\right)=\frac{\exp(\theta_p)}{1+\exp(\theta_p)}, 
\end{equation}
so that
\begin{equation} 
p_0 = \frac{1}{1+\exp(\theta_p)}. 
\end{equation}
Then the cumulant generating function becomes $F(\theta_p)=\log\left(1+\exp(\theta_p)\right)$.
By Lemma~\ref{lemma:holderclosedform},
\begin{align}
D^\Holder_\alpha(p(x):q(x))
&= \frac{1}{\alpha} \log\left(1+\exp(\alpha\theta_p)\right)
+ \frac{1}{\beta} \log\left(1+\exp(\beta\theta_q)\right)
- \log\left(1+\exp(\theta_p+\theta_q)\right),\\
D^\Holder_{\alpha,\gamma}(p(x):q(x))
&= \frac{1}{\alpha} \log\left(1+\exp(\gamma\theta_p)\right)
+ \frac{1}{\beta} \log\left(1+\exp(\gamma\theta_q)\right)
- \log\left(1+\exp\left(\frac{\gamma}{\alpha}\theta_p+\frac{\gamma}{\beta}\theta_q\right)\right).
\end{align}

\subsection{Case study: MultiVariate Normal distributions (MVNs)}

Let us now instantiate the formulas  for multivariate normals (Gaussian distributions).
We have the log-normalizer $F(\theta)$ expressed using the usual parameters as~\cite{SM-2011}:
\begin{equation}
F(\theta)=F(\mu(\theta),\Sigma(\theta))=\frac{1}{2} \log (2\pi)^d \vert\Sigma\vert + \frac{1}{2}\mu^\top\Sigma^{-1}\mu.
\end{equation}
Since
\begin{equation}
\theta=(\Sigma^{-1}\mu, -\frac{1}{2}\Sigma^{-1})=(v,M),\quad \mu=-\frac{1}{2}M^{-1}v,\quad \Sigma=-\frac{1}{2}M^{-1}.
\end{equation}

It follows that:

\begin{equation}
\theta_p+\theta_q=\theta_{p+q}
=(v_p+v_q,M_p+M_q)=\left(\Sigma^{-1}_p\mu_p+\Sigma^{-1}_q\mu_q,-\frac{1}{2}\Sigma^{-1}_p-\frac{1}{2}\Sigma^{-1}_q\right).
\end{equation}

Therefore, we have:

\begin{equation}
\mu_{p+q}=(\Sigma_p^{-1}+ \Sigma_q^{-1})^{-1}(\Sigma^{-1}_p\mu_p+\Sigma^{-1}_q\mu_q),\quad \Sigma_{p+q}=(\Sigma_p^{-1}+ \Sigma_q^{-1})^{-1}
\end{equation}

\begin{figure}[t]
\includegraphics[width=\textwidth]{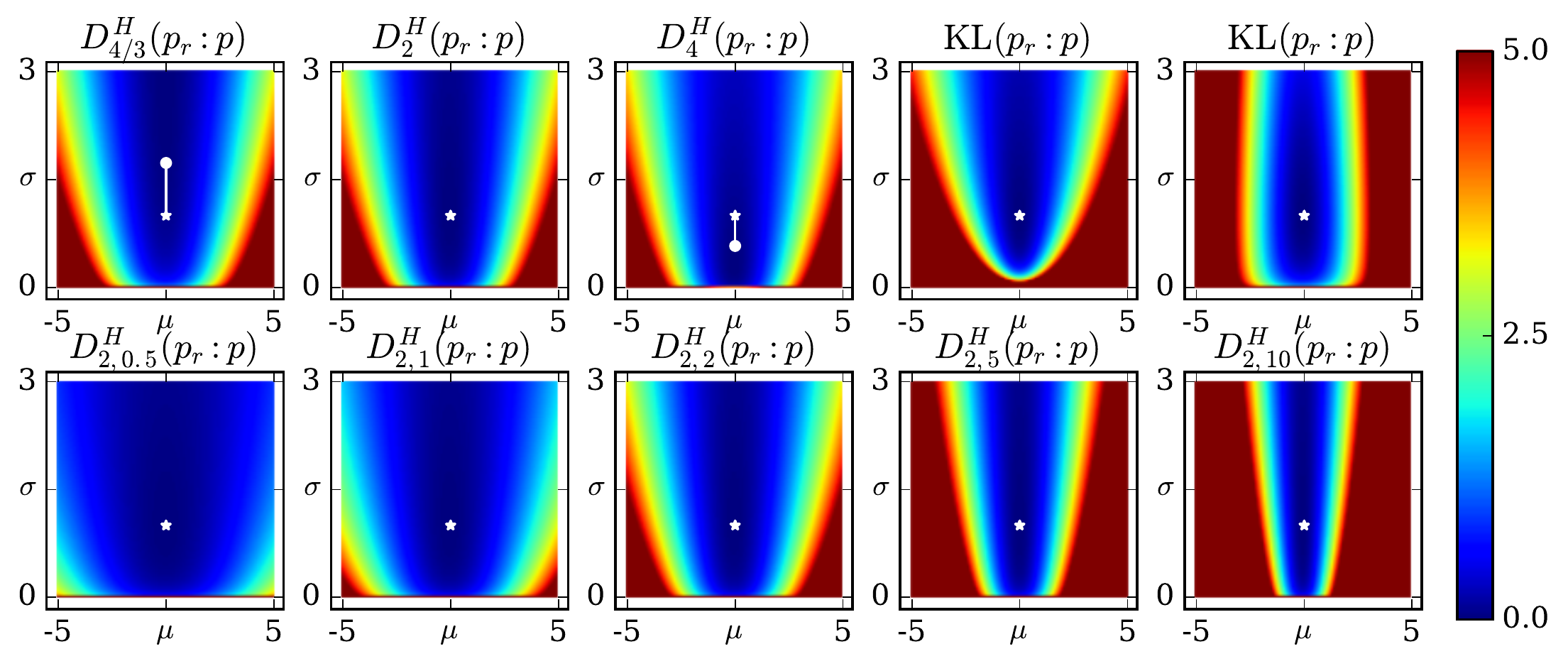}
\caption{First row: $D^\Holder_\alpha(p_r:p)$, where $p_r$ is standard Gaussian distribution, and $\alpha\in\{4/3,2,4\}$
as compared to the KL divergence. Second row: $D^\Holder_{2,\gamma}(p_r:p)$ for $\gamma\in\{0.5,1,2,5,10\}$.
Notice that $D^\Holder_{2,2}=D^\Holder_2$.}\label{fig:gauss}
\end{figure}

We thus get the following closed-form formula for $p~\sim N(\mu_p,\Sigma_p)$ and $q~\sim N(\mu_q,\Sigma_q)$:

\begin{align}
&D^\Holder_\alpha(N(\mu_p,\Sigma_p):N(\mu_q,\Sigma_q))=
\frac{1}{2\alpha}\log\left\vert\frac{\Sigma_p}{\alpha}\right\vert + \frac{1}{2}\mu_p^\top\Sigma_p^{-1}\mu_p
+
\frac{1}{2\beta}\log\left\vert\frac{\Sigma_q}{\beta}\right\vert + \frac{1}{2}\mu_q^\top\Sigma_q^{-1}\mu_q\nonumber\\
&\hspace{5em}+
\frac{1}{2}\log\left\vert\Sigma_p^{-1}+\Sigma_q^{-1}\right\vert - \frac{1}{2}
(\Sigma^{-1}_p\mu_p+\Sigma^{-1}_q\mu_q)^\top(\Sigma_p^{-1}+ \Sigma_q^{-1})^{-1}(\Sigma^{-1}_p\mu_p+\Sigma^{-1}_q\mu_q);\\
&D^\Holder_{\alpha,\gamma}(N(\mu_p,\Sigma_p):N(\mu_q,\Sigma_q))
=
\frac{1}{2\alpha}\log\left\vert\frac{\Sigma_p}{\gamma}\right\vert + \frac{\gamma}{2\alpha}\mu_p^\top\Sigma_p^{-1}\mu_p
+
\frac{1}{2\beta}\log\left\vert\frac{\Sigma_q}{\gamma}\right\vert + \frac{\gamma}{2\beta}\mu_q^\top\Sigma_q^{-1}\mu_q\nonumber\\
&\hspace{5em}+
\frac{1}{2}\log\left\vert\frac{\gamma}{\alpha}\Sigma_p^{-1}+\frac{\gamma}{\beta}\Sigma_q^{-1}\right\vert - \frac{1}{2}
\left(\frac{\gamma}{\alpha}\Sigma^{-1}_p\mu_p+\frac{\gamma}{\beta}\Sigma^{-1}_q\mu_q\right)^\top
\left(\frac{\gamma}{\alpha}\Sigma_p^{-1}+ \frac{\gamma}{\beta}\Sigma_q^{-1}\right)^{-1}
\left(\frac{\gamma}{\alpha}\Sigma^{-1}_p\mu_p+\frac{\gamma}{\beta}\Sigma^{-1}_q\mu_q\right).
\end{align}

Figure \ref{fig:gauss} shows HPD and HD for univariate Gaussian distributions as
compared to the KL divergence.

\subsection{Case study: Zero-centered Laplace distribution}

The zero-centered Laplace distribution is defined on the support $(-\infty,\infty)$
with the pdf:
\begin{equation}
p(x; \sigma)
=\frac{1}{2\sigma}\exp\left(-\frac{\vert{x}\vert}{\sigma}\right)
=\exp\left(-\frac{\vert{x}\vert}{\sigma}-\log(2\sigma)\right).
\end{equation}

We have $\theta=-\frac{1}{\sigma}$, $F(\theta)=\log(-\frac{2}{\theta})$.
Therefore, it comes that:
\begin{align}
D^\Holder_\alpha(p(x):q(x)) &=
\frac{1}{\alpha} \log\left(-\frac{2}{\alpha\theta_p}\right)
+ \frac{1}{\beta} \log\left(-\frac{2}{\beta\theta_q}\right)
- \log\left(-\frac{2}{\theta_p+\theta_q}\right)\nonumber \\
&=
\frac{1}{\alpha} \log\left(\frac{\sigma_p}{\alpha}\right)
+ \frac{1}{\beta} \log\left(\frac{\sigma_q}{\beta}\right)
- \log\left(\frac{\sigma_p\sigma_q}{\sigma_p+\sigma_q}\right),\\
D^\Holder_{\alpha,\gamma}(p(x):q(x)) &=
\frac{1}{\alpha} \log\left(-\frac{2}{\gamma\theta_p}\right)
+ \frac{1}{\beta} \log\left(-\frac{2}{\gamma\theta_q}\right)
- \log\left(-\frac{2}{\frac{\gamma}{\alpha}\theta_p+\frac{\gamma}{\beta}\theta_q}\right) \nonumber\\
&=
\frac{1}{\alpha} \log\sigma_p + \frac{1}{\beta} \log\sigma_q
-\log\left(\frac{\sigma_p\sigma_q}{\frac{1}{\beta}\sigma_p+\frac{1}{\alpha}\sigma_p}\right).
\end{align}
In this special case, 
$D^\Holder_{\alpha,\gamma}(p(x):q(x))$ does not vary with $\gamma$.

\subsection{Case study: Wishart distribution}

The Wishart distribution is defined on the $d\times{d}$ positive definite
cone with the density:
\begin{equation}
p(X; n, S) =
\frac{ \vert{X}\vert^{\frac{n-d-1}{2}} \exp\left(-\frac{1}{2}tr(S^{-1}X)\right)}
{2^{\frac{nd}{2}}\vert{S}\vert^{\frac{n}{2}}\Gamma_d\left(\frac{n}{2}\right)},
\end{equation}
where $n>d-1$ is the degree of freedom, and $S\succ0$ is a scale matrix.
We rewrite it into the canonical form:
\begin{equation}
p(X; n, S) =
\exp\left( -\frac{1}{2}tr(S^{-1}X)
+\frac{n-d-1}{2}\log\vert{X}\vert
-\frac{nd}{2}\log2
-\frac{n}{2}\log\vert{S}\vert
-\log\Gamma_d\left(\frac{n}{2}\right)
\right).
\end{equation}
We can see that $\theta=(\theta^1,\theta^2)$,
$\theta^1=-\frac{1}{2}S^{-1}$, $\theta^2=\frac{n-d-1}{2}$,
and
\begin{align}
F(\theta)
&=\frac{nd}{2}\log2 +\frac{n}{2}\log\vert{S}\vert +\log\Gamma_d\left(\frac{n}{2}\right)\nonumber\\
&=(\theta^2+\frac{d+1}{2})d\log2
+(\theta^2+\frac{d+1}{2})\log\bigg\vert-\frac{1}{2}(\theta^1)^{-1}\bigg\vert
+\log\Gamma_d\left(\theta^2+\frac{d+1}{2}\right).
\end{align}
The resulting $D^\Holder_\alpha(p(x):q(x))$ and
$D^\Holder_{\alpha,\gamma}(p(x):q(x))$ are straightforward from the
above expression of $F(\theta)$ and Lemma~\ref{lemma:holderclosedform}.
We will omit these tedious expressions for brevity.

\subsection{Approximating H\"older projective divergences for statistical mixtures}\label{sec:hmm}

Given two finite mixture models $m(x)=\sum_{i=1}^k w_i p_i(x)$ and $m'(x)=\sum_{j=1}^{k'}
w_j' p_j'(x)$, we  derive analytic bounds of their H\"older divergences.
When  approximation is  only needed, one may compute H\"older divergences based
on Monte-Carlo stochastic sampling.

Let us assume that all mixture components are in an exponential family~\cite{EF-2009} so that
$p_i(x)=p(x;\theta_i)=\exp(\inner{\theta_i}{t(x)} - F(\theta_i))$
and 
$p_j'(x)=p(x;\theta_j')=\exp(\inner{\theta_j'}{t(x)} - F(\theta_j'))$ are densities (wrt the Lebesgue measure $\mu$).

We rewrite the H\"older divergence into the form:
\begin{equation}\label{eq:holdermixture}
D^\Holder_{\alpha}(m(x):m'(x))
=
-\log\int_\calX m(x)m'(x) \dx
+\frac{1}{\alpha}\log\int_\calX m(x)^\alpha \dx
+\frac{1}{\beta}\log\int_\calX m'(x)^\beta \dx.
\end{equation}

To compute the first term, we observe that a product of mixtures is also a mixture:
\begin{align}
\int_\calX m(x) m'(x) \dx
&=\sum_{i=1}^k \sum_{j=1}^{k'} k_i k'_j
\int_\calX p_i(x) p_j'(x) \dx \nonumber\\
&=\sum_{i=1}^k \sum_{j=1}^{k'} k_i k'_j
\int_\calX \exp\left( \inner{\theta_i+\theta'_j}{t(x)} - F(\theta_i) - F(\theta_j') \right) \dx \nonumber\\
&=\sum_{i=1}^k \sum_{j=1}^{k'} k_i k'_j
\exp\left(F(\theta_i)+F(\theta_j')-F(\theta_i+\theta_j')\right),
\end{align}
which can be computed in $O(kk')$ time.

The second and third  terms in Eq. \ref{eq:holdermixture} are not straightforward to calculate and shall be bounded.
Based on computational geometry,
we adopt the log-sum-exp bounding technique of~\cite{klmm} and divide the
support $\calX$ into $L$ pieces of elementary intervals $\calX=\biguplus{}_{l=1}^{L}I_l$.
In each interval $I_l$, the indices:
\begin{equation}
\delta_l   = \argmax_{i} w_ip_i(x)
\text{  and  }
\epsilon_l = \argmin_{i} w_ip_i(x),
\end{equation}
representing the unique dominating component and the dominated component. Then we bound as follows:
\begin{align}\label{eq:lse}
\max\left\{
\int_{I_l} k^\alpha w_{\epsilon_l}^\alpha p_{\epsilon_l}(x)^\alpha \dx,
\int_{I_l} w_{\delta_l}^\alpha p_{\delta_l}(x)^\alpha \dx \right\}
\le 
\int_{I_l} m(x)^\alpha \dx
\le
\int_{I_l} k^\alpha w_{\delta_l}^\alpha p_{\delta_l}(x)^\alpha \dx.
\end{align}

All terms on the lhs and rhs of Eq. \ref{eq:lse} can be computed exactly by noticing that:

\begin{equation}
\int_{I} p_i(x)^\alpha dx = \int_{I} \exp( \inner{\alpha \theta_i}{t(x)} - \alpha F(\theta_i) )
= \exp( F(\alpha\theta_i) - \alpha F(\theta_i) ) \int_I p(x; \alpha\theta_i) \dx.
\end{equation}

When $\alpha\theta\in\Theta$ where $\Theta$ denotes the natural parameter space, the integral $\int_I p(x; \alpha\theta_i) \dx$ converges, see~\cite{klmm} for further details.

Then the bounds of $\int_\calX m(x)^\alpha \dx$ can be obtained by summing 
the bounds in Eq. \ref{eq:lse} over all elementary intervals.
Thus $D^\Holder_{\alpha,\beta}(m(x):m'(x))$ can be both lower and upper bounded.

\section{H\"older centroids and center-based clustering}\label{sec:centroid}

We study the application of HPD and HD for clustering distributions~\cite{gammadivclustering-2014},
specially clustering Gaussian distributions~\cite{GClust-2000,ClustNormal-2006,ClustNormal-2009},
which have been used in sound processing~\cite{GClust-2000}, sensor network~\cite{ClustNormal-2006}, statistical debugging~\cite{ClustNormal-2006}, 
and quadratic invariants of switched systems~\cite{quadinvariant-2015}, etc.
Other potential applications of HD may include nonnegative matrix factorization~\cite{NMFbeta-2014},
and clustering von Mises-Fisher~\cite{vMF-2005,vMF-2014} (log-normalizer expressed using Bessel functions).

 
\subsection{H\"older centroids}

We study center-based clustering of a finite set of distributions:
Given a list of  distributions belonging to the same conic exponential family with natural parameters
$\{\theta_1,\cdots,\theta_n\}$ and their associated positive weights $\{w_1,\cdots,w_n\}$
with $\sum_{i=1}^n w_i=1$, consider their centroids based on HPD and HD as follows:
\begin{align}
C_{\alpha}(\{\theta_i,w_i\})        &= \argmin_{C} \sum_{i=1}^n w_i D^\Holder_\alpha(\theta_i:C),\\
C_{\alpha,\gamma}(\{\theta_i,w_i\}) &= \argmin_{C} \sum_{i=1}^n w_i D^\Holder_{\alpha,\gamma}(\theta_i:C).
\end{align}

By abuse of notation, $C$ denotes both the HPD centroid and HD centroid.
When the context is clear, the parameters in parentheses can be omitted so that these centroids
are simply denoted as $C_{\alpha}$ and $C_{\alpha,\gamma}$.
Both of them are defined as the right-sided centroids.
The corresponding left-handed centroids are obtained according
to reference duality, i.e.,
\begin{align}
C_{\baralpha}        &= \argmin_{C} \sum_{i=1}^n w_i D^\Holder_\alpha(C:\theta_i),\\
C_{\baralpha,\gamma} &= \argmin_{C} \sum_{i=1}^n w_i D^\Holder_{\alpha,\gamma}(C:\theta_i).
\end{align}

By Lemma~\ref{lemma:holderclosedform}, these centroids can be obtained for distributions belonging to the same exponential family as follows:
\begin{align}
C_{\alpha}
&= \argmin_{C}\left[
\frac{1}{\beta} F(\beta C) - \sum_{i=1}^n w_i F(\theta_i+C) \right],\\
C_{\alpha,\gamma}
&= \argmin_{C}\left[
\frac{1}{\beta}F(\gamma C) - \sum_{i=1}^n w_i F\left(\frac{\gamma}{\alpha}\theta_i+\frac{\gamma}{\beta}C\right) \right].
\end{align}

Let $\gamma=\alpha$, we get:
\begin{equation}
C_{\alpha,\alpha}(\{\theta_i,w_i\}) = \argmin_{C}\left[
\frac{1}{\beta}F(\alpha C) - \sum_{i=1}^n w_i F\left(\theta_i+ \frac{\alpha}{\beta} C\right)
\right]
= \frac{\beta}{\alpha} C_\alpha = \frac{1}{\alpha-1} C_{\alpha}(\{\theta_i,w_i\}),
\end{equation}
meaning that the HPD centroid is just a special case of HD centroid up to
a {\em scaling transformation} in the natural parameters space.
Let $\gamma=\beta$, we get:
\begin{equation}
C_{\alpha,\beta}(\{\theta_i,w_i\}) = \argmin_{C}\left[
\frac{1}{\beta}F(\beta C) - \sum_{i=1}^n w_i F\left(\frac{\beta}{\alpha}\theta_i + C\right)
\right]
= C_\alpha\left(\left\{\frac{\beta}{\alpha}\theta_i,w_i\right\}\right)
= C_\alpha\left(\left\{\frac{1}{\alpha-1}\theta_i,w_i\right\}\right).
\end{equation}

Let us consider the general HD centroid $C_{\alpha,\gamma}$.
Since $F$ is convex, the minimization energy is the sum of a {\em convex function}
$\frac{1}{\beta} F(\gamma C)$ with a {\em concave function} $ -\sum_{i=1}^n w_i F\left(\frac{\gamma}{\alpha}\theta_i+\frac{\gamma}{\beta}C\right)$.
We can therefore use the concave-convex procedure (CCCP)~\cite{BR-2011} that optimizes Difference of Convex Programs (DCPs):
We start with $C_{\alpha,\gamma}^0=\sum_{i=1}^n w_i \theta_i$ (the barycenter, belonging to $\Theta$) and then update:
\begin{equation}
C_{\alpha,\gamma}^{t+1} = \frac{1}{\gamma} (\nabla F)^{-1}
\left(
\sum_{i=1}^n w_i \nabla F\left(\frac{\gamma}{\alpha}\theta_i+\frac{\gamma}{\beta}C_{\alpha,\gamma}^t\right) \right).
\end{equation}
for $t=0,1,\cdots$ until convergence. This can be done by noting that $\eta=\nabla F(\theta)$ 
are the dual parameters that are also known as the expectation parameters (or moment parameters).
Therefore $\nabla{F}$ and $(\nabla{F})^{-1}$ can be computed through Legendre
transformations between the natural parameter space and the dual parameter space.

This iterative optimization is guaranteed to converge to a {\em local} minimum,
with a main advantage of bypassing the learning rate parameter of gradient descent algorithms.
Since $F$ is strictly convex, $\nabla F$ is monotonous, and the rhs expression can
be interpreted as a multi-dimensional quasi-arithmetic mean.
In fact, it is a barycenter on non-normalized weights scaled by $\beta=\baralpha$.

For exponential families, the symmetric HPD centroid is:
\begin{equation}
O_{\alpha}
= \argmin_O \sum_{i=1}^n w_i S^\Holder_{\alpha}(\theta_i:O)
= \argmin_O 
\left[ \frac{1}{2\alpha}F(\alpha{O}) +\frac{1}{2\beta}F(\beta{O})
-\sum_{i=1}^n w_i F(\theta_i+O) \right].
\end{equation}

In this case, CCCP update rule is not in closed form because we
cannot easily inverse the sum of gradients (but when $\alpha=\beta$, the two terms
collapse so CS centroid can be calculated using CCCP).
Nevertheless, we can implement the reciprocal operation numerically.
Interestingly, the symmetric HD centroid can be solved by CCCP!
It amounts to solve:
\begin{equation}
O_{\alpha,\gamma}
= \argmin_O \sum_{i=1}^n w_i S^\Holder_{\alpha,\gamma}(\theta_i:O)
= \argmin_O
\left[
F(\gamma O) - \sum_{i=1}^n w_i \left(
F\left(\frac{\gamma}{\alpha}\theta_i+\frac{\gamma}{\beta}O\right) + F\left(\frac{\gamma}{\beta}\theta_i+\frac{\gamma}{\alpha}O\right)
\right)
\right].
\end{equation}

One can apply CCCP to iteratively update the centroids based on:
 \begin{equation}
O_{\alpha,\gamma}^{t+1} = \frac{1}{\gamma} (\nabla F)^{-1}
\left[ \sum_{i=1}^n w_i \left(
\frac{1}{\beta}
\nabla F
\left(\frac{\gamma}{\alpha}\theta_i+\frac{\gamma}{\beta}O_{\alpha,\gamma}^t\right)
+ \frac{1}{\alpha}
\nabla F
\left(\frac{\gamma}{\beta}\theta_i+\frac{\gamma}{\alpha}O_{\alpha,\gamma}^t\right)
\right)\right].
\end{equation}

Notice the similarity with the updating procedure of $C_{\alpha,\gamma}^t$.

Once the centroid, say $O_{\alpha,\gamma}$, has been computed, we calculate the associated {\em H\"older information}:
\begin{equation}
\sum_{i=1}^n w_i S^\Holder_{\alpha,\gamma}(\theta_i:O_{\alpha,\gamma}).
\end{equation}

The H\"older information generalizes the notion of variance and Bregman information~\cite{cb-2005} to the case of H\"older distances.

\subsection{Clustering based on symmetric H\"older divergences}
Given a set of fixed densities $\{p_1,\cdots,p_n\}$, we can perform
{\em variational $k$-means}~\cite{tJ-2015} with
respect to the H\"older divergence to minimize the cost function:
\begin{equation}
E(O_1,\cdots,O_L,l_1,\cdots,l_n)
= \sum_{i=1}^n S^\Holder_{\alpha,\gamma}(p_i:O_{l_i}),
\end{equation}
where $O_1,\cdots,O_L$ are the cluster centers, and
$l_i\in\{1,\cdots,L\}$ is the cluster label of $p_i$.
The algorithm is given by Algorithm~\ref{algo:kmeans}.
Notice that one does not need to wait for the 
CCCP iterations to converge. It only has to improve
the cost function $E$ before updating the assignment.
We have implemented the algorithm based on the symmetric
HD. One can easily modify it based on HPD and other variants.
 
\begin{algorithm}
\DontPrintSemicolon
\KwIn{A list of probability distributions $p_1,\cdots,p_n$; number of clusters $L$; $\alpha>1$; $\gamma>0$}
\KwOut{A clustering scheme $p_i\to\{1,\cdots,L\}$, $\forall{i}\in\{1,\cdots,n\}$}
Randomly assign $p_i$ to $l_i$\;
\While{not converged}{
   \For{$l=1,\dots,L$}{
	  \tcc{Variational $k$-means: Carry CCCP iterations until the current center improves the former cluster H\"older information} 
       Compute the centroid $O_l= \arg\min_O\sum_{i:l_i=l} S^\Holder_{\alpha,\gamma}(p_i:O)$\;
   }
   \For{$i=1,\dots,n$}{
   	   Assign $l_i=\arg\min_{l}S^\Holder_{\alpha,\gamma}(p_i:O_l)$ \;
   }
}
\Return{$\{l_i\}_{i=1}^n$}\;
\caption{H\"older variational $k$-means}\label{algo:kmeans}
\end{algorithm}

\begin{figure}[ht!]
\centering
\begin{subfigure}[b]{.7\textwidth}
\includegraphics[width=\textwidth]{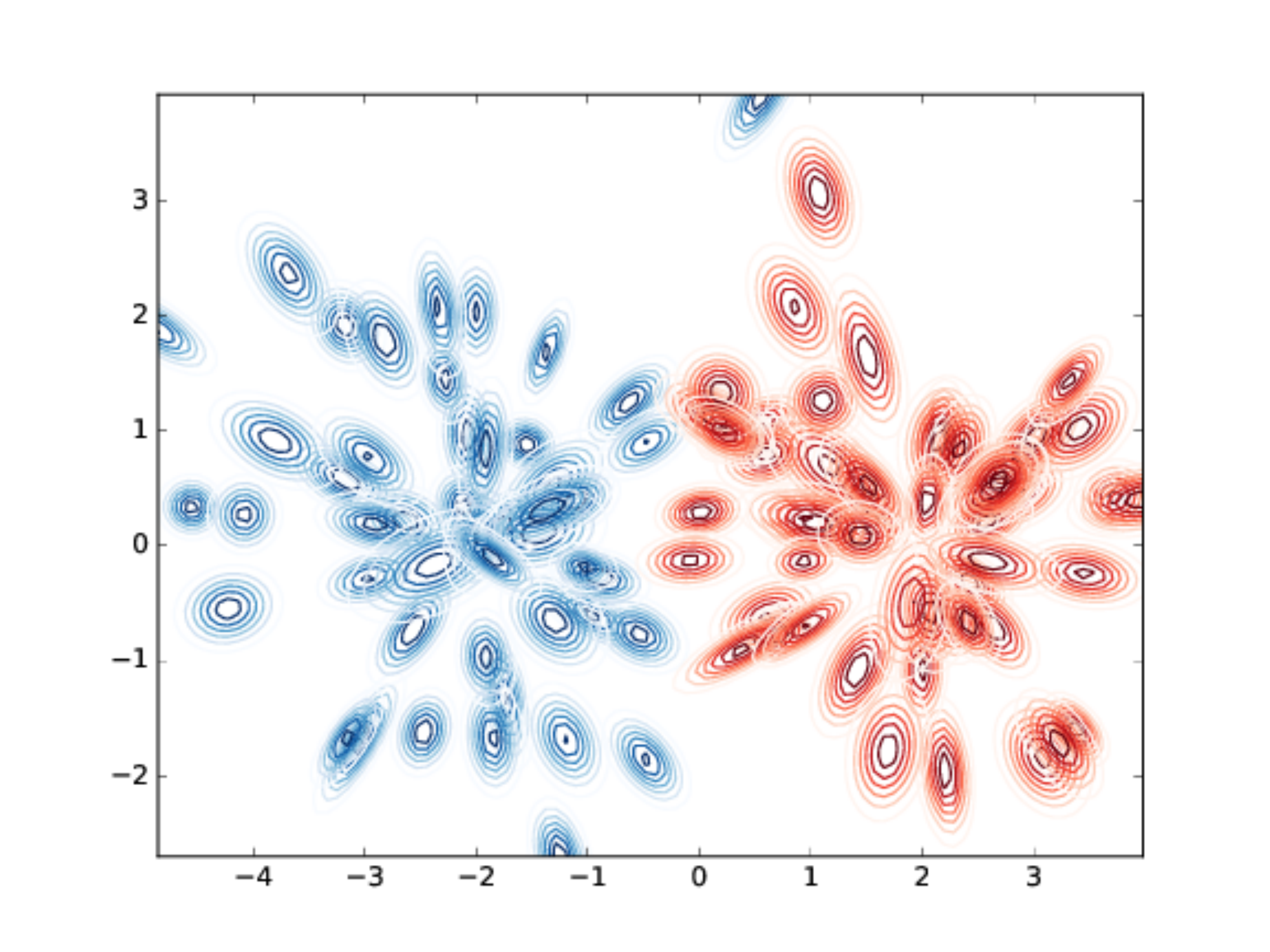}
\caption{$\alpha=\gamma=1.1$ (H\"older clustering)}
\end{subfigure}
\begin{subfigure}[b]{.7\textwidth}
\includegraphics[width=\textwidth]{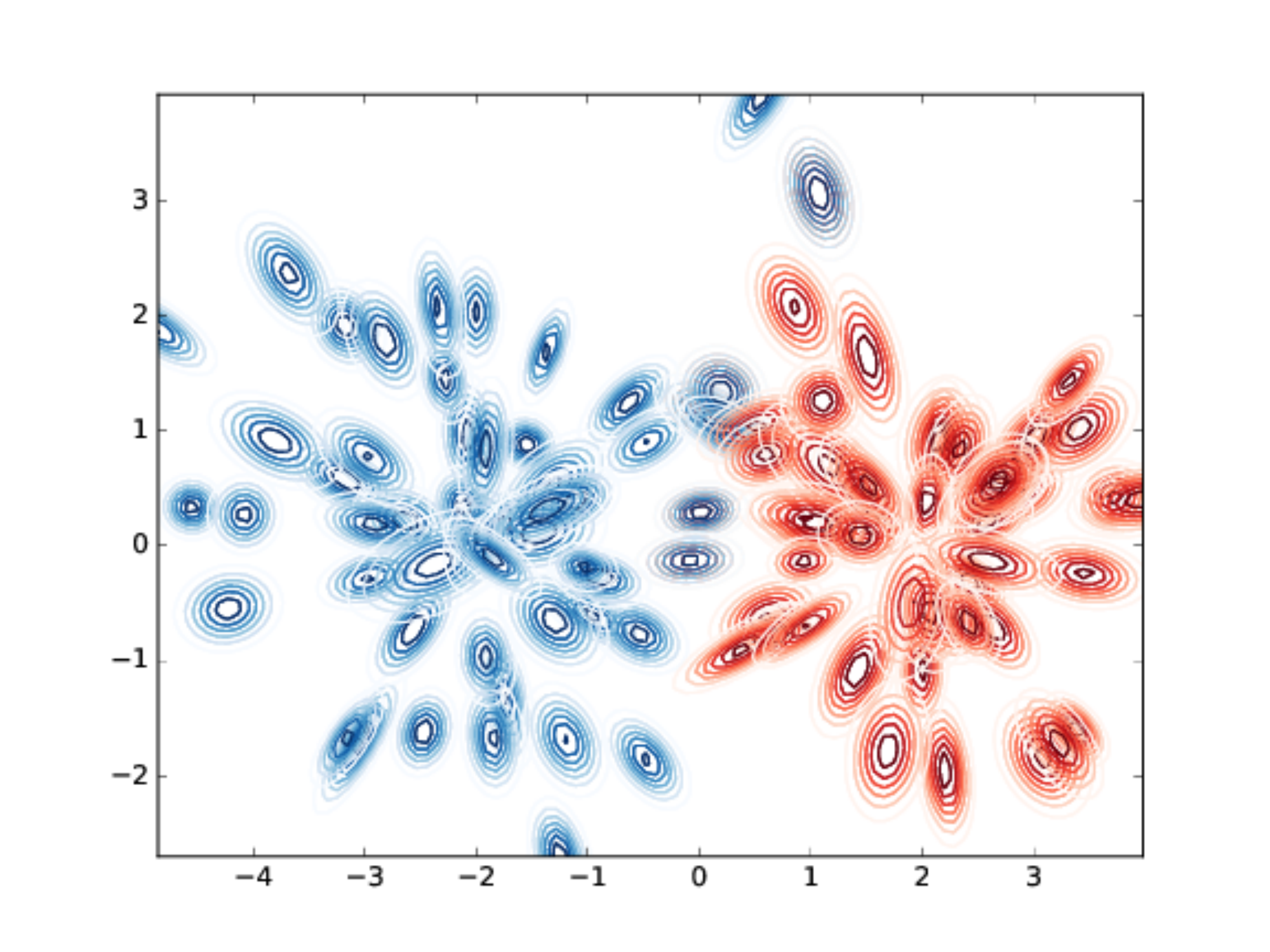}
\caption{$\alpha=\gamma=2$ (Cauchy-Schwarz clustering)}
\end{subfigure}
\caption{Variational $k$-means clustering results on
a toy dataset consisting of a set of 2D Gaussians organized into two clusters.}\label{fig:cluster}
\end{figure}

\begin{table}[ht]
\centering
\caption{Clustering accuracy of the 2D Gaussian dataset (based on 500 independent runs).}\label{tbl:acc}
\begin{tabular}{l|cccc}
\hline
Parameters & $\alpha=\gamma=1.1$ & $\alpha=\gamma=1.5$ & $\alpha=\gamma=2$ (CS) & $\alpha=\gamma=10$\\
\hline
$n=50$   & $\mathbf{95.6\%\pm7.6\%}$ & $92.4\pm9.8\%$ & $92.2\%\pm9.9\%$  & $92.2\%\pm9.7\%$ \\
$n=100$  & $\mathbf{97.3\%\pm5.3\%}$ & $94.9\pm8.8\%$ & $94.0\%\pm10.1\%$ & $94.2\%\pm9.3\%$ \\
\hline
\end{tabular}
\end{table}

We made a toy dataset generator which can randomly generate $n$ 2D Gaussians
which have a underlying structure of two clusters.
In the first cluster, the mean of each Gaussian $G(\mu, \Sigma)$ has the prior distribution
$\mu\sim G((-2,0), I)$; the covariance matrix is obtained by first
generating $\sigma_1\sim\Gamma(7,0.01)$, $\sigma_2\sim\Gamma(7,0.003)$,
where $\Gamma$ means a gamma distribution with prescribed shape and scale,
then rotating the covariance matrix $\diag(\sigma_1,\sigma_2)$ so that the resulting
Gaussian has a ``radial direction'' with respect to the center $(-2,0)$.
The second cluster is similar
to the first cluster with the only difference being that 
$\mu\sim G((2,0), I)$ is centered around $(-2,0)$.
See Fig.~\ref{fig:cluster} for an intuitive presentation of the toy dataset.

To reduce the number of parameters that has to be tuned,
we only investigate the case $\alpha=\gamma$.
If we choose $\alpha=\gamma=2$, then 
$S^\Holder_{\alpha,\gamma}$ becomes the CS divergence
and Algorithm~\ref{algo:kmeans} reduces to traditional CS clustering.
From Fig.~\ref{fig:cluster} we can observe that the clustering
result does vary with the settings of $\alpha$ and $\gamma$.
We performed clustering experiments on two different settings of the sample size.
Table~\ref{tbl:acc} shows the clustering accuracy measured by
the percentage of ``correctly clustered'' Gaussians, i.e.,
the output label by clustering algorithms that coincides
with the true label corresponding to the data generating process.
We see that the symmetric H\"older divergence can give better
clustering results as compared to CS clustering. This hints that
one could consider the general H\"older divergence to replace CS in
similar clustering applications~\cite{CSShapeMatching-2014,CSGGTexture-2016}. 
Although one faces the problem of
tuning the parameter $\alpha$ and $\gamma$, H\"older divergences
can potentially give {\em better} results. This is expected because
CS is just one particular case of the class of H\"older divergences.

\section{Conclusion and perspectives}\label{sec:concl}

We introduced the notion of pseudo-divergences that generalizes the concept of divergences in information geometry~\cite{IG-2016} that are smooth non-metric statistical distances that are not required to obey the law of the indiscernibles.
Pseudo-divergences can be built from inequalities by considering the inequality difference gap or its log-ratio gap. 
We then defined two classes of statistical measures based on H\"older's ordinary and reverse inequalities: 
The singly parametric family of H\"older pseudo-divergences and the doubly parametric family of H\"older divergences.
By construction, the H\"older divergences are proper divergences between probability densities.
Both statistical H\"older distance families  are projective divergences that do not require distributions to be normalized, and admit closed-form expressions when considering exponential families with conic or affine natural parameter space (like multinomials or multivariate normals). Those two families of distances can be symmetrized and intersect at the unique Cauchy-Schwarz divergence.
Since the Cauchy-Schwarz divergence is often used in distribution clustering applications~\cite{CSShapeMatching-2014}, we carried out preliminary experiments demonstrating experimentally that thesymmetrized H\"older divergences improved over the Cauchy-Schwarz divergence  for a toy dataset of Gaussians.
We briefly touched upon the use of these novel divergences in statistical estimation theory.
These projective  H\"older (pseudo-)divergences are different from the recently introduced compositive scored-induced H\"older  divergences~\cite{affineInvariantDivergence-2014,scaleinvariantDiv-2014} that are not projective divergences and do not admit closed-form expressions for exponential families in general.

We elicited the special role of escort distributions~\cite{IG-2016} for H\"older divergences  in our framework:
Escort distributions transform distributions to allow one:
\begin{itemize}
\item To reveal that H\"older pseudo-divergences on escort distributions amount to skew Bhattacharyya divergences~\cite{BR-2011},
\item To transform the improper  H\"older pseudo-divergences into proper H\"older divergences, and vice versa.
\end{itemize}

Let us conclude with a perspective note on pseudo-divergences, statistical estimators, and manifold learning.
Proper divergences have been widely used in statistical estimators to build families of estimators~\cite{Pardo-2005,SD-2011}.
Similarly, given a prescribed density $p_0(x)$, a pseudo-divergence yields a corresponding estimator by minimizing $D(p_0(x):q(x))$
with respect to $q(x)$. However in this case the resulting $q(x)$ is potentially biased and is not
guaranteed to recover the optimal input $p_0(x)$. Furthermore, the minimizer of
$D(p_0(x):q(x))$ may not be unique, i.e., there could be more than one probability density $q(x)$ yielding $D(p_0(x):q(x))=0$.

\emph{How can pseudo-divergences be useful?} 
We have the following two simple arguments:

\begin{itemize}
\item In an estimation scenario, we can usually pre-compute $p_1(x)\neq{p_0}(x)$ according to
$D(p_1(x):p_0(x))=0$. Then the estimation $q(x)=\arg\min_q D(p_1(x):q(x))$ will automatically target at $p_0(x)$.
We call this technique \emph{pre-aim}.

For example, given positive measure $p(x)$, we first find $p_0(x)$ to satisfy
$D^\Holder_\alpha(p_0(x):p(x))=0$. We have
$p_0(x)=p(x)^{\frac{1}{\alpha-1}}$  that satisfies this condition.
Then a proper divergence between $p(x)$ and $q(x)$ 
can be obtained by aiming $q(x)$ towards $p_0(x)$.
For conjugate exponents $\alpha$ and $\beta$,
\begin{align}
D^\Holder_{\alpha}(p_0(x):q(x))
&= -\log\frac{\int_\calX p_0(x)q(x) \dx}
{\left(\int_\calX p_0(x)^{\alpha} \dx\right)^{1/\alpha}
\left(\int_\calX q(x)^{\beta} \dx\right)^{1/\beta}},\\
&= -\log\frac{\int_\calX p(x)^{\frac{1}{\alpha-1}} q(x) \dx}
{\left(\int_\calX p(x)^{\frac{\alpha}{\alpha-1}} \dx\right)^{1/\alpha}
\left(\int_\calX q(x)^{\beta} \dx\right)^{1/\beta}},\\
&= -\log\frac{\int_\calX p(x)^{\frac{\beta}{\alpha}} q(x) \dx}
{\left(\int_\calX p(x)^{\beta} \dx\right)^{1/\alpha}
\left(\int_\calX q(x)^{\beta} \dx\right)^{1/\beta}}
= D^\Holder_{\alpha,\beta}(p(x):q(x)).
\end{align}
This means that the pre-aim technique of HPD is equivalent to HD
$D^\Holder_{\alpha,\gamma}$ when we set $\gamma=\beta$.

As an alternative implementation of pre-aim, since $D^\Holder_\alpha(p(x):p^{\alpha-1}(x))=0$, 
a proper divergence between $p(x)$ and $q(x)$ can be constructed by measuring:
\begin{equation}
D^\Holder_{\alpha}(q(x):p^{\alpha-1}(x))
= -\log\frac{\int_\calX p(x)^{\frac{\alpha}{\beta}} q(x) \dx}
{\left(\int_\calX q(x)^{\alpha} \dx\right)^{1/\alpha}
\left(\int_\calX p(x)^{\alpha} \dx\right)^{1/\beta}}
= D^\Holder_{\beta,\alpha}(p(x):q(x)),
\end{equation}
turning out again to belong to the class of HD.

In practice, HD as a two-parameter family is less used than HPD with pre-aim
because of the difficulty in choosing the parameter $\gamma$, and because that
HD has a slightly more complicated expression. The family of HD 
connecting CS divergence with skew Bhattacharyya divergence~\cite{BR-2011}
is nevertheless of theoretical importance.

\item In manifold learning~\cite{hinton03,ee,igml}, it is an essential
topic to align two category distributions $p_0(x)$ and $q(x)$ corresponding
respectively to the input and output~\cite{igml}, both for learning and for performance evaluation.
In this case, the dimensionality of the statistical manifold that encompasses $p_0(x)$ and $q(x)$
is so high that to preserve monotonically $p_0(x)$ in the resulting $q(x)$ is already a difficult
non-linear optimization and could be sufficient for the application, while preserving perfectly
the input $p_0(x)$ is not so meaningful because of the input noise. It is then much easier to define
pseudo-divergences using inequalities which are not necessarily to be proper with potentially more choices. 
On the other hand, projective divergences including H\"older divergences introduced in this work are
more meaningful in manifold learning than KL divergence (which is widely used) because they give scale
invariance of the probability densities, meaning that one can define positive similarities
then directly align these similarities, which is guaranteed to be equivalent to
align the corresponding distributions. This could potentially give unified perspectives
in between the two approaches of similarity based manifold learning~\cite{ee} and the probabilistic approach~\cite{hinton03}.

\end{itemize}

We expect that these two novel parametric H\"older classes of statistical divergences and pseudo-divergences open up new insights and applications in statistics and information sciences. Furthermore, the framework to build divergences or pseudo-divergences from proper or improper biparametric inequalities offers novel classes of divergences to study.

Reproducible source code is available online at:\\
 \url{https://www.lix.polytechnique.fr/\textasciitilde nielsen/HPD/}


\appendix

\section{Proof of H\"older ordinary and reverse inequalities}\label{sec:hdproof}
We extend the proof (\cite{cvx}, pp.78) to prove both the (ordinary or forward) H\"older
inequality and the reverse H\"older inequality.

\begin{proof}
First, let us observe that $-\log(x)$ is strictly convex on $(0,+\infty)$ since $(-\log(x))''=\frac{1}{x^2}$.
It follows that for $0< a <1$ that:

\begin{equation}
-\log( a x_1 + (1-a) x_2 ) \le -a\log(x_1) - (1-a)\log(x_2),
\end{equation}
where the equality holds iff $x_1=x_2$.

Conversely, when $a<0$ or $a>1$, we have:
\begin{equation}
-\log( a x_1 + (1-a) x_2 ) \ge -a\log(x_1) - (1-a)\log(x_2),
\end{equation}
where the equality holds iff $x_1=x_2$.

Equivalently, we can write these two inequalities as follows:
\begin{equation}\label{eq:cvx}
\left\{\begin{array}{ll}
x_1^a x_2^{1-a} & \le ax_1 + (1-a)x_2\quad(\text{if }0<a<1);\\
x_1^a x_2^{1-a} & \ge ax_1 + (1-a)x_2\quad(\text{if }a<0\text{ or }a>1),
\end{array}\right.
\end{equation}
both of them are tight iff $x_1=x_2$.

Let $P$ and $Q$ be positive measures with Radon-Nikodym densities $p(x)>0$ and $q(x)>0$ be positive densities with respect to the reference Lebesgue measure $\mu$.
The densities are strictly greater than $0$ on the support $\calX$.
Plugging:
\begin{equation}
a=\frac{1}{\alpha},\quad
1-a=\frac{1}{\beta},\quad
x_1 = \frac{p(x)^\alpha}{\int_{\calX} p(x)^\alpha \dx},\quad
x_2 = \frac{q(x)^\beta}{\int_{\calX} q(x)^\beta \dx},\quad
\end{equation}
into Eq. \ref{eq:cvx}, we get:
\begin{equation}
\left\{
\begin{array}{ll}
\frac{p(x)}{ \left(\int_{\calX} p(x)^\alpha \dx \right)^{1/\alpha} }
\frac{q(x)}{ \left(\int_{\calX} q(x)^\beta  \dx \right)^{1/\beta} }
\le 
\frac{1}{\alpha} \frac{p(x)^\alpha}{\int p(x)^\alpha \dx} +
\frac{1}{\beta} \frac{q(x)^\beta}{\int q(x)^\beta \dx} &
\text{if $\alpha>0$ and $\beta>0$,}\\
\frac{p(x)}{ \left(\int_{\calX} p(x)^\alpha \dx \right)^{1/\alpha} }
\frac{q(x)}{ \left(\int_{\calX} q(x)^\beta  \dx \right)^{1/\beta} }
\ge 
\frac{1}{\alpha} \frac{p(x)^\alpha}{\int p(x)^\alpha \dx} +
\frac{1}{\beta} \frac{q(x)^\beta}{\int q(x)^\beta \dx} &
\text{if $\alpha<0$ or $\beta<0$.}
\end{array}
\right.
\end{equation}
Assume that $p(x)$ in $L_\alpha(\calX,\mu)$ and $q(x)$ in $L_\beta(\calX,\mu)$ so that both $\int_{\calX} p(x)^\alpha \dx$ 
and $\int_{\calX} q(x)^\beta  \dx$ converge.
Integrate both sides on $\calX$ to get:
\begin{equation}
\left\{
\begin{array}{ll}
\frac{\int_{\calX} p(x) q(x) \dx}
{ \left(\int_{\calX} p(x)^\alpha \dx \right)^{1/\alpha}
  \left(\int_{\calX} q(x)^\beta  \dx \right)^{1/\beta} } \le 1 &
\text{if $\alpha>0$ and $\beta>0$,}\\
\frac{\int_{\calX} p(x) q(x) \dx}
{ \left(\int_{\calX} p(x)^\alpha \dx \right)^{1/\alpha}
  \left(\int_{\calX} q(x)^\beta  \dx \right)^{1/\beta} } \ge 1 &
\text{if $\alpha<0$ or $\beta<0$.}
\end{array}\right.
\end{equation}

The necessary and sufficient condition for equality is that: 
\begin{equation}
\frac{p(x)^\alpha}{\int_{\calX} p(x)^\alpha \dx}
=
\frac{q(x)^\beta}{\int_{\calX} q(x)^\beta \dx},
\end{equation}
almost everywhere.
That is, there exists a positive constant $\lambda>0$   such that:
\begin{equation}
p(x)^\alpha = \lambda q(x)^\beta, \quad \lambda>0,\quad \mbox{almost everywhere.}
\end{equation}

\end{proof}

The H\"older conjugate exponents $\alpha$ and $\beta$ satisfies $\frac{1}{\alpha}+\frac{1}{\beta}=1$.
That is $\beta=\frac{\alpha}{\alpha-1}$. Thus when $\alpha<0$ we necessarily have $\beta>0$, and vice-versa.

We can unify these two straight and reverse H\"older inequalities into a single inequality by considering the sign of $\alpha\beta=\frac{\alpha^2}{\alpha-1}$:
We get the general H\"older inequality:

\begin{equation}
(-1)^{\sign(\alpha\beta)} \frac{\int_{\calX} p(x) q(x) \dx}
{ \left(\int_{\calX} p(x)^\alpha \dx \right)^{1/\alpha}
  \left(\int_{\calX} q(x)^\beta  \dx \right)^{1/\beta} } \geq (-1)^{\sign(\alpha\beta)}.
\end{equation}

When $\alpha=\beta=2$, H\"older inequality becomes the Cauchy-Schwarz inequality:\footnote{Historically, Cauchy stated the discrete sum inequality in 1821 while Schwarz reported the integral form of the inequality in 1888.} 
\begin{equation}
\int_\calX p(x)q(x)\dx \leq \sqrt{\left(\int_\calX p(x)^2 \dx\right)\left(\int_\calX q(x)^2 \dx\right)}.
\end{equation}

\section{Notations}

\begin{tabular}{ll}
$D_\alpha^\HS$ &   H\"older proper non-projective Scored-induced divergence~\cite{scaleinvariantDiv-2014}\\
$D_{\alpha}^\Holder$ & H\"older improper projective pseudo-divergence (new)\\
$D_{\alpha,\gamma}^\Holder$ & H\"older proper projective divergence (new)\\
$D_\alpha^\HE$ & H\"older proper projective escort divergence (new)\\
$\KL$ & Kullback-Leibler divergence~\cite{cover-2012}\\
$\CS$ & Cauchy-Schwarz divergence~\cite{CSRepDataSampling-2011}\\
$B$ & Bhattacharyya distance~\cite{Bhatt-1943}\\
$B_{\frac{1}{\alpha}}$ & skew Bhattacharyya distance~\cite{BR-2011}\\
$D_\gamma$ & $\gamma$-divergence (score-induced)~\cite{GammaDiv-2008}\\
$p_\alpha^E, q_\beta^E$ & escort distributions\\
$\alpha, \beta$ & H\"older conjugate pair of exponents: $\frac{1}{\alpha}+\frac{1}{\beta}=1$ \\
$\bar\alpha,\beta$ & H\"older conjugate exponent: $\bar\alpha=\beta=\frac{\alpha}{\alpha-1}$ \\
$\theta_p, \theta_q$ & natural parameters of exponential family distributions\\
$\calX$ & support of distributions\\
$\mu$ & Lebesgue measure\\
$L^\gamma(\calX,\mu)$ & Lebesgue space of functions $f$ such that $\int_\calX |f(x)|^\gamma \dx<\infty$ 
\end{tabular}

\tableofcontents

\end{document}